\documentclass{article}

\usepackage{microtype}
\usepackage{graphicx}
\usepackage{booktabs}
\usepackage{makecell}

\usepackage{hyperref}

\usepackage[preprint]{icml2026}

\usepackage{amsmath}
\usepackage{amssymb}
\usepackage{mathtools}
\usepackage{amsthm}

\usepackage[capitalize,noabbrev]{cleveref}

\theoremstyle{plain}
\newtheorem{theorem}{Theorem}[section]
\newtheorem{proposition}[theorem]{Proposition}
\newtheorem{lemma}[theorem]{Lemma}
\newtheorem{corollary}[theorem]{Corollary}
\theoremstyle{definition}

\theoremstyle{remark}

\icmltitlerunning{Remoteness-Aware Control of Negative Off-Policy Updates}
\begin{document}

\twocolumn[
  \icmltitle{Breaking the Curse of Repulsion: Remoteness-Aware Control of Negative Off-Policy Updates}



  \begin{icmlauthorlist}
    \icmlauthor{Yusen Huo}{comp}
    \icmlauthor{Changping Wang}{comp}
    \icmlauthor{Yangru Huang}{comp}
    \icmlauthor{Jun Zhang}{comp}
    \icmlauthor{Jie Jiang}{comp}
  \end{icmlauthorlist}
  \icmlaffiliation{comp}{Tencent Inc, China}

  \icmlcorrespondingauthor{Jun Zhang}{neoxzhang@tencent.com}

  \icmlkeywords{Off-Policy Reinforcement Learning, Negative Feedback, Policy Optimization, Stability}

  \vskip 0.3in
]




\printAffiliationsAndNotice{}

\begin{abstract}
Off-policy policy optimization reuses historical behavior, including
negative-advantage samples that suppress known failures. We show that repeated
reuse can turn this useful signal into excessive repulsion: as the learner
moves away from a historical negative action, subsequent updates make that
action increasingly remote without necessarily reducing its update strength.
Our aggregate theory characterizes the resulting transition from a stable
displacement beyond the positive-only target to persistent drift and the loss
of finite stable equilibria; controlled strength sweeps show that an
intermediate displacement can improve held-out reward. The relevant
learner-relative coordinate is squared standardized distance for Gaussian
policies and surprisal for categorical policies. We introduce Dynamic
Remoteness-Aware Policy Optimization (DRPO), which leaves the negative update
unchanged in the near field and exponentially attenuates its remote tail.
DRPO restores eventual inward Gaussian drift for every fixed finite
negative-to-positive mass ratio, yields an explicit ultimate-bound radius, and
changes categorical support suppression from exponential to polynomial
probability decay. External diagnostics and controlled interventions isolate
remoteness-dependent policy geometry as a source of negative-update
amplification and show that selective tapering can remove its far-field effect
without discarding useful local feedback.
\end{abstract}

\section{Introduction}
\label{sec:introduction}

Policy-based methods learn from signed feedback: positive updates attract the
policy toward actions estimated to be preferable, whereas negative updates
suppress actions estimated to be undesirable. Negative feedback can therefore encode useful boundaries and exclusion
information that positive-only learning lacks
\citep{zhu2025negative,song2025good}. Such information can displace the
policy beyond the solution supported by positive feedback alone, although
whether that displacement improves task utility is an empirical question. In off-policy settings, however,
the same historical sample may be reused after earlier updates have already
changed its probability under the current policy. This raises a single
question: when does repeated reuse turn useful negative feedback into
excessive repulsion?

We identify learner-relative remoteness as the state variable that links one
reuse step to the next. A negative update lowers the probability of its historical action,
increasing the action's remoteness and thereby changing the next
policy-score response (the \emph{near field} \(D\leq\tau\) versus the
\emph{far field} \(D>\tau\); Section~\ref{sec:drpo}). The policy family determines the resulting tail
law: Gaussian mean scores grow with standardized distance, whereas categorical
logit scores remain bounded even as selected-action probabilities approach the
simplex boundary. Under isolated reuse, uncontrolled Gaussian remoteness grows
geometrically and categorical action probability decays exponentially; an
exponential taper reduces the corresponding far-field growth to logarithmic
order. This is the curse of repulsion: following a negative update makes its next
reuse more remote without making the signal self-limiting.

What remains unresolved is whether remoteness itself amplifies negative
updates independently of sample quality. Existing methods control off-policy
instability through filtering, importance weighting, probability-aware gates,
and tapered updates
\citep{novati2019remember,fakoor2020p3o,peng2019advantage,
schulman2017proximal,fujimoto2019off,kumar2020conservative,
arnal2025asymmetric,leroux2025tapered}, but remoteness is usually confounded
with advantage magnitude and data quality. We therefore match those factors
while varying only remoteness, then intervene separately on near- and far-field
negative updates to test their downstream effects.

Our aggregate analysis places these per-sample effects in competition with
positive attraction. When positive mass dominates, repulsion can displace the
policy beyond the positive-only target while preserving a finite stable
equilibrium. At balance, the restoring force disappears and persistent drift
emerges; under negative dominance, no finite stable equilibrium remains.
These regimes characterize policy geometry and stability, not task utility,
which must be determined empirically.

Motivated by this transition, we introduce Dynamic Remoteness-Aware Policy
Optimization (DRPO). DRPO recomputes remoteness under the current policy,
leaves the negative branch unchanged in the near field, and applies an
exponential taper beyond the threshold. The taper dominates every finite-order
far-field response, restores eventual inward Gaussian drift for any fixed
finite negative-to-positive mass ratio, and makes categorical suppression
self-limiting. Remoteness controls geometric exposure, not directional utility
(Section~\ref{subsec:distance_dependent_utility}).

We first verify the pattern externally, then use controlled environments to
isolate its source, test its downstream effect, and compare control rules.

Our contributions are: \textbf{(1)} we characterize how repeated negative
reuse changes policy-relative remoteness and derive aggregate regimes for
stable extrapolation, persistent drift, and the loss of finite stable
equilibria; \textbf{(2)} we derive a control-order hierarchy and a selective
exponential taper with Gaussian ultimate-boundedness and categorical
self-limiting suppression guarantees; and \textbf{(3)} matched external
diagnostics and controlled interventions separate coefficient magnitude
from policy geometry, identify a far-field causal pathway, and test the
stability--utility trade-off of selective tapering.

% ------------------------------------------------------------------
% SECTION 2: PRELIMINARIES
% ------------------------------------------------------------------
\section{Related Work}
\label{sec:related_work}

\textbf{Off-policy historical reuse across policy domains.}
Off-policy policy optimization reuses data generated by earlier policies,
from replay and offline-RL datasets in continuous control to fixed
preference and verifier-labeled corpora in language-model post-training.
Classical approaches correct behavior--learner mismatch or constrain policy
support through importance weighting, clipping, trust regions, and
conservative or behavior-regularized objectives
\citep{precup2000eligibility,schulman2015trust,
schulman2017proximal,fujimoto2019off,kumar2020conservative,
fujimoto2021minimalist,kostrikov2021offline}. Truncated importance ratios also underlie corrected off-policy return
estimators such as Retrace and V-trace
\citep{munos2016safe,espeholt2018impala}; their primary role is
distribution correction rather than learner-relative negative-tail control. Across both continuous and
discrete policies, our focus is narrower: repeated reuse of signed actor
feedback, especially the negative branch, after the learner has moved away
from the historical action.

\textbf{Controlling negative updates.}
Existing controls span positive filtering, which removes the negative branch
\citep{peng2019advantage,nair2020awac,kostrikov2021offline}; fixed global
down-weighting of negative advantages, as in HPO's hysteretic weight
\citep{sana2026hpo}; baseline-based positive--negative rebalancing in
AsymRE \citep{arnal2025asymmetric}; reference-relative preference
optimization in DPO \citep{rafailov2023direct}; and probability-aware
tapering in TOPR \citep{leroux2025tapered}. Failed or rejected behavior is
also exploited in unlikelihood and preference-based training
\citep{welleck2019neural,ethayarajh2024kto,duan2024negating,
zhang2024negative,zhu2025negative,song2025good}.
DRPO differs by treating learner-relative remoteness as a dynamical state
created by repeated reuse and by selectively attenuating only the far-field
negative tail.
% ------------------------------------------------------------------
% SECTION 2: PRELIMINARIES
% ------------------------------------------------------------------

\section{Problem Setup}
\label{sec:problem_setup}

% ------------------------------------------------------------------
% SECTION 2: PRELIMINARIES
% ------------------------------------------------------------------

\subsection{Off-Policy Signed Actor Updates}

We consider policy optimization from a historical update distribution
$\nu$ over state--action pairs, which may represent an offline dataset,
a replay buffer, or trajectories collected by earlier or stale behavior
policies. Let $\pi_\theta(a \mid s)$ denote the current policy and let
$\widehat{A}(s,a)$ denote the advantage-like signal used by the actor
update. During an actor step, the update distribution and the associated
advantage estimates are treated as fixed with respect to $\theta$. The resulting empirical signed actor field has the score-function form
\citep{williams1992simple,sutton2000policy}:
\begin{equation}
    \mathbf{F}(\theta)
    =
    \mathbb{E}_{(s,a)\sim\nu}
    \left[
        \widehat{A}(s,a)
        \nabla_\theta \log \pi_\theta(a \mid s)
    \right],
    \label{eq:empirical_actor_field}
\end{equation}
with the corresponding update
\(\theta_{t+1}=\theta_t+\eta\mathbf{F}(\theta_t)\).
Repeated reuse means that the same historical samples and signed feedback
are evaluated under an evolving current policy;
Section~\ref{sec:theory} studies the actor-side dynamics induced by this
reuse.

To separate favorable and unfavorable feedback, write
\(\widehat A^{+}=\max\{\widehat A,0\}\) and
\(\widehat A^{-}=\max\{-\widehat A,0\}\), so the actor field decomposes as
\begin{equation}
\begin{aligned}
    \mathbf{F}(\theta)
    ={}&
    \mathbb{E}_{\nu}
    \left[
        \widehat{A}^{+}(s,a)
        \nabla_\theta \log \pi_\theta(a \mid s)
    \right] \\
    &-
    \mathbb{E}_{\nu}
    \left[
        \widehat{A}^{-}(s,a)
        \nabla_\theta \log \pi_\theta(a \mid s)
    \right].
\end{aligned}
\label{eq:signed_actor_decomposition}
\end{equation}
For a negative sample, the update magnitude factorizes as
$\widehat{A}^{-}(s,a)
\lVert\nabla_\theta \log \pi_\theta(a \mid s)\rVert$,
separating the severity of its estimated disadvantage from its
learner-relative policy geometry.

We use \emph{coefficient magnitude} for \(\widehat A^-\), \emph{score
response} for \(\|\nabla\log\pi\|\), \emph{update magnitude} for their
product, \emph{reuse loop gain} for the squared remoteness gradient
(Section~\ref{sec:distance_strength}), and \emph{aggregate force} for the
expectation or sum of signed updates.

All aggregate quantities use the same update base measure \(\nu\). We write
\begin{equation}
    p=\mathbb E_\nu[\widehat A^+],
    \qquad
    q=\mathbb E_\nu[\widehat A^-],
    \qquad
    \rho=q/p \quad (p>0).
    \label{eq:aggregate_masses}
\end{equation}
Because \(p\) and \(q\) are population expectations under \(\nu\), positive
and negative sample frequencies are already included.

% ------------------------------------------------------------------
% SECTION 3: THEORETICAL FRAMEWORK
% ------------------------------------------------------------------
\section{Far-Field Dynamics of Historical Policy Updates}
\label{sec:theory}

We first characterize how remoteness controls score contribution in
policy-output coordinates. We then show how reuse creates positive
self-attenuation and negative self-amplification, and derive the resulting
aggregate regimes and policy-family manifestations.

\subsection{Learner-Relative Remoteness and Update Strength}
\label{sec:distance_strength}

To isolate the action-side geometry, we fix a context $s$ and write
$u$ for the policy-output coordinates at that context. For a historical
action $a$, we define its learner-relative remoteness as
\begin{equation}
    D_u(a)
    =
    -\log \pi_u(a).
    \label{eq:learner_relative_remoteness}
\end{equation}
For a categorical policy, $D_u(a)$ is the surprisal of the selected
action. For a Gaussian policy, it equals a scale-dependent constant plus
one half of the squared Mahalanobis distance from the policy mean.
Thus, $D_u(a)$ measures how unlikely the historical action is under the
current learner, rather than defining a metric in the action space.

We measure the corresponding squared score norm by
\begin{equation}
    R_u(a)
    =
    \left\|
        \nabla_u \log \pi_u(a)
    \right\|_2^2.
    \label{eq:score_strength}
\end{equation}
A sample with advantage estimate $\widehat{A}(s,a)$ contributes an
output-space update of magnitude
$|\widehat{A}(s,a)|\sqrt{R_u(a)}$.
The following results characterize how learner-relative remoteness
controls this update strength.

\begin{proposition}[Distance-dependent score-function magnitude]
\label{prop:distance_strength}
Fix a context and a historical action \(a\).

For a Gaussian policy with mean \(\mu\) and fixed covariance
\(\Sigma\succ0\), let
\(C_\Sigma=\frac{d}{2}\log(2\pi)+\frac12\log|\Sigma|\) denote the
Gaussian normalization constant, which is independent of \(a\) and
\(\mu\). Then learner-relative remoteness satisfies
\begin{equation}
    D_\mu(a)-C_\Sigma
    =
    \frac12
    (a-\mu)^\top\Sigma^{-1}(a-\mu).
    \label{eq:gaussian_remoteness_mahalanobis}
\end{equation}
Moreover, the squared mean-score norm
\(R_\mu(a)=\|\nabla_\mu\log\pi_\mu(a)\|_2^2\) is bounded above and below
by positive constants times \(D_\mu(a)-C_\Sigma\), with constants
determined only by \(\Sigma\). Consequently, as standardized distance
increases, the Gaussian mean-score response grows without bound. When
\(\Sigma=\sigma^2I\), \(R_\mu(a)=2(D_\mu(a)-C_\Sigma)/\sigma^2\), so
remoteness and squared mean-score norm have the same global ordering.

For a categorical policy with logits \(z\), learner-relative remoteness is
the selected-action surprisal \(D_z(a)=-\log\pi_z(a)\). The selected-logit
score response is
\begin{equation}
    \left|
    \frac{\partial}{\partial z_a}
    \log \pi_z(a)
    \right|
    =
    1-e^{-D_z(a)}.
    \label{eq:categorical_selected_score}
\end{equation}
Thus, categorical surprisal can diverge as the selected probability
approaches zero, but the selected-logit response saturates at one. The
full logit-score norm remains bounded as well.
\end{proposition}

(Proof: Appendix~\ref{app:distance_strength}.) Gaussian remoteness can
produce an unbounded mean-score response, whereas categorical surprisal can
diverge while the logit-score response remains bounded.

\subsection{Self-Amplification under Historical Reuse}
\label{sec:self_amplification}

The preceding result compares samples at a fixed policy. We now consider
a fixed historical action that is repeatedly reused while the learner
changes. Write
\[
    D(u)=D_u(a),
    \qquad
    R(u)=\|\nabla_uD(u)\|_2^2,
\]
and consider the repeated single-sample update
\begin{equation}
    u_{t+1}
    =
    u_t
    +
    \eta\widehat A\nabla_u\log\pi_{u_t}(a)
    =
    u_t
    -
    \eta\widehat A\nabla_uD(u_t).
    \label{eq:single_sample_reuse_update}
\end{equation}
Here, the action \(a\) and its signed coefficient \(\widehat A\) are
held fixed over the reuse window.

\begin{theorem}[Self-attenuation and self-amplification under reuse]
\label{thm:reuse_dynamics}
Assume that \(D(u)\) is differentiable and convex on a convex
neighborhood containing the update segments, and define
\[
    D_t=D(u_t),
    \qquad
    R_t=R(u_t).
\]
If \(\widehat A<0\), then every reuse step satisfies
\begin{equation}
    D_{t+1}
    \ge
    D_t+\eta|\widehat A|R_t,
    \qquad
    R_{t+1}\ge R_t.
    \label{eq:negative_reuse_amplification}
\end{equation}
Thus, negative reuse makes the historical action increasingly remote
while preventing its score-function norm from decreasing.

If \(\widehat A>0\), and \(D\) additionally has an \(L\)-Lipschitz
gradient on the same neighborhood with \(0<\eta\widehat A\le 1/L\), then
\begin{equation}
    D_{t+1}
    \le
    D_t-\frac{\eta\widehat A}{2}R_t,
    \qquad
    R_{t+1}\le R_t.
    \label{eq:positive_reuse_attenuation}
\end{equation}
Thus, positive reuse makes the historical action increasingly compatible
with the learner while attenuating its subsequent score response.
\end{theorem}

Both policy coordinates satisfy the theorem's convexity assumption
globally: fixed-covariance Gaussian remoteness is quadratic in the mean,
and categorical negative log-likelihood is convex in logits
(Appendix~\ref{app:reuse_dynamics}).

The theorem separates a one-time distant update from a historical-reuse
feedback loop. Combined with Proposition~\ref{prop:distance_strength}, it
yields unbounded amplification for fixed-covariance Gaussian mean updates
and bounded but persistent amplification for categorical logits.

For an isolated isotropic-Gaussian negative sample, repeated
uncontrolled reuse produces geometric growth in remoteness; the exact
rate is given in Appendix~\ref{app:gaussian_reuse_rate}.

\subsection{Aggregate Equilibria under Historical Reuse}
\label{sec:aggregate_equilibria}

The preceding single-sample reuse analysis tracks one fixed historical
action. We now aggregate positive and negative historical samples and
ask whether their competing attraction and repulsion produce a finite
stable equilibrium, persistent drift, or instability.

To analyze continuous-action Gaussian and discrete-action categorical
policies within a single argument, we use a common exponential-family
output representation at a fixed context; the exact coordinates and
regularity assumptions are given in
Appendix~\ref{app:aggregate_equilibrium}. Using the aggregate masses
\(p,q,\rho\) from Equation~\eqref{eq:aggregate_masses}, let
\(\mathbf m_+\) denote the positive-only target, the mean-parameter point
selected by positive feedback alone. When \(q>0\), let \(\mathbf m_-\)
denote the magnitude-weighted negative-sample moment (a repulsion moment,
not an attraction target).

\begin{theorem}[Aggregate attraction--repulsion regimes]
\label{thm:aggregate_equilibria}
Consider the continuous aggregate actor dynamics and the corresponding
discrete update under the regular minimal exponential-family
representation specified in
Appendix~\ref{app:aggregate_equilibrium}. Assume \(p>0\).

\textbf{Positive-only case \((\rho=0)\).}
If \(\mathbf m_+\) lies in the interior of the attainable mean-parameter
space, there is a unique finite equilibrium whose mean parameter is
\(\mathbf m_+\). It is locally asymptotically stable in continuous time
and for sufficiently small discrete step sizes.

\textbf{Positive-dominant regime \((0<\rho<1)\).}
Define the signed target
\begin{equation}
    \mathbf m^\star(\rho)
    =
    \mathbf m_+
    +
    \frac{\rho}{1-\rho}
    \left(
        \mathbf m_+-\mathbf m_-
    \right).
    \label{eq:signed_target_ratio}
\end{equation}
If \(\mathbf m^\star(\rho)\) lies in the interior of the attainable
mean-parameter space, there is a unique finite equilibrium whose mean
parameter is \(\mathbf m^\star(\rho)\). It is locally asymptotically
stable in continuous time and for sufficiently small discrete step
sizes. Holding \(\mathbf m_+\) and \(\mathbf m_-\) fixed, its displacement
from \(\mathbf m_+\) grows without bound as \(\rho\to1^-\), unless
\(\mathbf m_+=\mathbf m_-\). If the signed target leaves the attainable
mean-parameter space, no finite equilibrium exists.

\textbf{Critical regime \((\rho=1)\).}
If \(\mathbf m_+\neq\mathbf m_-\), the policy-output coordinate exhibits
persistent drift and admits no finite equilibrium. If
\(\mathbf m_+=\mathbf m_-\), every point is stationary, but no point is
an isolated asymptotically stable equilibrium.

\textit{Negative-dominant regime (\(\rho>1\)).}
Every finite stationary point, if one exists, is unstable in continuous
time and under every discrete update with positive step size. Hence, this
regime admits no locally asymptotically stable finite equilibrium.
\end{theorem}

The proof, including the exponential-family actor field and the exact
discrete-time step-size condition, is provided in
Appendix~\ref{app:aggregate_equilibrium}.
Equation~\eqref{eq:signed_target_ratio} gives the central geometric
interpretation: while positive feedback remains dominant, negative
feedback displaces the equilibrium from \(\mathbf m_+\) in the direction
\(\mathbf m_+-\mathbf m_-\), and therefore away from
\(\mathbf m_-\).

% ------------------------------------------------------------------
% SECTION 4: METHODOLOGY
% ------------------------------------------------------------------
\subsection{Policy-Family Manifestations of Negative Dominance}
\label{sec:policy_family_manifestations}

Negative dominance admits no finite stable equilibrium, but its score
and support manifestations depend on the policy family. We therefore
analyze Gaussian and categorical policies separately.

\begin{theorem}[Policy-family score behavior under negative dominance]
\label{thm:policy_family_manifestations}
Assume \(p<q\).

\textbf{Gaussian policy.}
Consider a Gaussian policy
\(\pi_\mu=\mathcal N(\mu,\Sigma)\) with fixed
\(\Sigma\succ0\). Its aggregate mean dynamics have a unique finite
stationary point, which is repelling. Unless initialized exactly at that
point, the distance from it diverges under both the continuous dynamics
and every discrete trajectory with fixed step size \(\alpha>0\).
Consequently, for every fixed historical action \(a\),
\begin{equation}
    \left\|
        \nabla_\mu\log\pi_\mu(a)
    \right\|_2
    =
    \left\|
        \Sigma^{-1}(a-\mu)
    \right\|_2
    \longrightarrow
    \infty.
    \label{eq:gaussian_unbounded_score}
\end{equation}

\textbf{Categorical policy.}
Consider a categorical policy
\(\pi_z=\operatorname{softmax}(z)\) with at least two actions and the
zero-sum gauge \(\mathbf 1^\top z=0\). Any finite stationary point, when
one exists, is unique and repelling. Unless initialized exactly at that
point, the continuous dynamics and every discrete trajectory with fixed
step size \(\alpha>0\) eventually leave every compact subset of the
logit space. Along some subsequence, the corresponding probability
vector approaches the boundary of the simplex. Nevertheless, for every
action \(a\) and every finite \(z\),
\begin{equation}
\|\nabla_z\log\pi_z(a)\|_2^2
=
\|e_a-\pi_z\|_2^2
\le 2(1-\pi_z(a))^2
\le 2.
\label{eq:categorical_bounded_score}
\end{equation}
\end{theorem}

The proof is provided in
Appendix~\ref{app:policy_family_manifestations}.
Fixed covariance is sufficient for the Gaussian result; learned covariance is
not required.

\subsection{Why Selective Tapering Works}
\label{sec:selective_tapering_bridge}

The preceding regimes identify the failure; the following radial limit
identifies the control property that reverses it.

\begin{lemma}[Far-field radial balance]
\label{lem:far_field_radial_balance}
For a fixed-covariance Gaussian field with finite negative atoms,
\[
\begin{aligned}
F(\mu)&=p\Sigma^{-1}(m_+-\mu)
 +\sum_iq_i\Sigma^{-1}(\mu-a_i),\\
q&=\sum_iq_i,
\end{aligned}
\]
and \(\mu=m_++rv\), \(\|v\|_2=1\),
\[
\lim_{r\to\infty}\frac{\langle v,F(\mu)\rangle}{r}
=(q-p)v^\top\Sigma^{-1}v.
\]
\end{lemma}

\begin{corollary}[Selective tapering versus global scaling]
\label{cor:selective_vs_global}
If each negative atom is multiplied by
\(\omega_i(D_i)\to0\), the limit in
Lemma~\ref{lem:far_field_radial_balance} becomes
\(-p v^\top\Sigma^{-1}v<0\). Thus, for every fixed finite
\(\rho=q/p\), without \(\rho\)-specific tuning, the far-field radial
component is eventually inward; the onset radius is not uniform in
\(\rho\) and depends on the tail rate of \(\omega\). If additionally
\(\omega(D)\sqrt D\to0\), the absolute tapered negative force vanishes.
In contrast, global scaling \(\omega\equiv\alpha\) is inward only when
\(\alpha q<p\). Proofs are in Appendix~\ref{app:far_field_bridge}.
\end{corollary}

\begin{proposition}[Restoration by reference regularization]
\label{prop:regularized_restoration}
The regularized linear field has a unique finite stable equilibrium when
\(p>q\) for every \(\gamma\ge0\), when \(p=q\) for every \(\gamma>0\),
and when \(p<q\) if and only if \(\gamma>q-p\). At \(p=q\), its
displacement is \(O(\gamma^{-1})\). Reference regularization restores
stability by tethering the policy; DRPO instead removes the far-field
source. The proof and the Euclidean-spring variant are in
Appendix~\ref{app:regularized_restoration}.
\end{proposition}

Corollary~\ref{cor:selective_vs_global} justifies selective tapering: it
removes the remote source, whereas a constant global scale must be tuned to
the negative-to-positive mass ratio. These conclusions assume frozen data,
fixed advantages, and unbounded reuse; data refresh and actor--critic
coupling alter the field.

\section{Dynamic Remoteness-Aware Policy Optimization}
\label{sec:drpo}

Section~\ref{sec:selective_tapering_bridge} identifies the design target:
control the remote negative source without uniformly weakening the negative
branch. We first define DRPO, then derive its control-order hierarchy and
closed-loop guarantees.

\subsection{Update and Policy-Family Instantiations}
\label{subsec:drpo_update}

For remoteness \(D\), DRPO uses the thresholded exponential taper
\begin{equation}
    \omega_{\mathrm{Exp}}(D)
    =
    \exp
    \left\{
        -\lambda
        \left[
            \frac{D-\tau}{c}
        \right]_{+}
    \right\},
    \label{eq:drpo_exp_weight}
\end{equation}
where \(\tau\) marks the control threshold, \(c>0\) fixes the remoteness
scale, and \(\lambda>0\) controls the decay rate. The weight is exactly one
when \(D\leq\tau\) and decays smoothly when \(D>\tau\).

Using \(D_\theta(s,a)=-\log\pi_\theta(a\mid s)\), define
\(\bar D_\theta(s,a)=\operatorname{sg}[D_\theta(s,a)]\), where
\(\operatorname{sg}[\cdot]\) denotes stop-gradient. DRPO changes only the
negative branch:
\begin{equation}
    \widetilde A_{\theta}(s,a)
    =
    \widehat A^{+}(s,a)
    -
    \omega_{\mathrm{Exp}}
    \!\left(\bar D_\theta(s,a)\right)
    \widehat A^{-}(s,a),
    \label{eq:drpo_effective_advantage}
\end{equation}
with actor field
\begin{equation}
    \mathbf F_{\mathrm{DRPO}}(\theta)
    =
    \mathbb E_{(s,a)\sim\nu}
    \left[
        \widetilde A_{\theta}(s,a)
        \nabla_\theta\log\pi_\theta(a\mid s)
    \right].
    \label{eq:drpo_actor_field}
\end{equation}
Remoteness is recomputed under the current policy but detached before it is
used as a coefficient; DRPO is therefore a dynamic weighting rule rather than
a differentiable distance regularizer.

At each policy iterate, tapering preserves the aggregate regime structure
after replacing the negative mass and moment by
\[
q_\omega=\mathbb E_\nu[\widehat A^-\omega],
\qquad
\mathbf M_{-,\omega}=\mathbb E_\nu[\widehat A^-\omega T(a)].
\]
For \(0\le q_\omega<p\), the controlled signed target is
\[
\mathbf m_\omega^\star
=
\frac{p\mathbf m_+-\mathbf M_{-,\omega}}{p-q_\omega},
\]
with \(\mathbf M_{-,\omega}=0\) when \(q_\omega=0\).

\paragraph{Gaussian policies.}
For a Gaussian policy with mean \(\mu_\theta(s)\) and fixed covariance
\(\Sigma\), define
\begin{equation}
    D_{\mathrm G,\theta}(s,a)
    =
    (a-\mu_\theta(s))^\top
    \Sigma^{-1}
    (a-\mu_\theta(s)).
    \label{eq:gaussian_squared_distance}
\end{equation}
Since \(D_\theta-C_\Sigma=\frac12D_{\mathrm G,\theta}\), the additive
constant and factor \(1/2\) can be absorbed into the threshold and scale:
\begin{equation}
    \omega_{\mathrm G}(s,a)
    =
    \exp
    \left\{
        -\lambda_{\mathrm G}
        \left[
            \frac{D_{\mathrm G,\theta}(s,a)-\tau_{\mathrm G}}
                 {c_{\mathrm G}}
        \right]_{+}
    \right\}.
    \label{eq:drpo_gaussian_weight}
\end{equation}

\paragraph{Categorical policies.}
For a categorical policy, substituting selected-action surprisal into
Equation~\eqref{eq:drpo_exp_weight} gives
\begin{equation}
    \omega_{\mathrm C}(s,a)
    =
    \min
    \left\{
        1,\,
        \left(e^{\tau}\pi_\theta(a\mid s)\right)^{\lambda/c}
    \right\}.
    \label{eq:drpo_categorical_weight}
\end{equation}
\paragraph{Where existing off-policy controls intervene.}
Existing controls intervene at different points without sharing a design
goal. Importance weighting changes the influence of historical samples,
PPO-style methods constrain movement relative to a refreshed anchor
\citep{schulman2017proximal}, and reference regularization tethers the
policy. DRPO shares the sample-influence intervention point with importance
weighting, but replaces behavior-determined attenuation with a
learner-relative control law: remoteness is recomputed under the current
policy, the threshold and decay rate are designed, and only the negative
branch is weighted. With likelihood-defined remoteness, the clipped behavior
ratio \(\min\{(\pi_\theta/\mu)^\beta,1\}\) is itself an exponential taper,
with canonical TOPR as the \(\beta=1\) behavior-threshold case
\citep{leroux2025tapered}. Compared with movement- and location-based
constraints, DRPO directly suppresses the remote negative source while
preserving positive and near-field feedback. Exact identities,
anchor-refresh effects, and scope conditions are given in
Appendix~\ref{app:topr_relation}.

\subsection{Loop-Gain Order}
\label{subsec:min_order_control}

The quantity that drives the next change in remoteness is not the
single-step parameter displacement but the squared score
\(R(D)=\|\nabla D\|_2^2\). We call this the reuse \emph{loop gain}.

\begin{proposition}[Gaussian loop-gain order]
\label{prop:loop_gain_order}
For a fixed-covariance Gaussian mean policy,
\(I_{\mathrm{loop}}(D):=R(D)=\Theta(D)\) in the far field. Hence
\(\omega(D)I_{\mathrm{loop}}(D)\) is bounded exactly at the matching
order \(\omega(D)=O(D^{-1})\), and vanishes when
\(\omega(D)=o(D^{-1})\). Equivalently, the matching order is
\(O(r^{-2})\) in standardized distance \(r\), since
\(D-C_\Sigma=r^2/2\). The proof is in
Appendix~\ref{app:min_order_control}.
\end{proposition}

At the exact matching rate \(\omega(D)=\Theta(D^{-1})\), the weighted
parameter force is \(\Theta(D^{-1/2})\to0\), but the weighted reuse loop
gain is \(\Theta(1)\). An isolated reused atom can therefore keep driving
remoteness outward at a constant order. For isolated Gaussian reuse,
this yields geometric growth without control, linear growth under
matching-order reciprocal tapering, and logarithmic growth under
DRPO's exponential taper
(Appendix~\ref{app:gaussian_reuse_rate}).

\paragraph{Utility boundary.}
\label{subsec:distance_dependent_utility}
Remoteness is a control coordinate rather than a utility label. When
historical negative feedback no longer supplies local correction, DRPO
targets \(\omega(D)I(D)\to0\) as \(D\to\infty\); task direction remains
supplied by the underlying advantage or verifier signal. The corresponding
utility formulation is given in
Appendix~\ref{app:negative_update_utility}.

\subsection{Guarantee Hierarchy}
\label{sec:exponential_taper}

The DRPO taper satisfies the stronger requirement above for every
finite-order response considered here.

\begin{proposition}[Vanishing finite-order far-field response]
\label{prop:exp_vanishing_influence}
Let \(I(D)\) denote the raw far-field response scale. If, for some
finite \(m\geq0\) and \(C>0\),
\(I(D)\leq C(1+D)^m\) in the far field, then the taper in
Equation~\eqref{eq:drpo_exp_weight} satisfies
\[
    \omega_{\mathrm{Exp}}(D) I(D)
    \longrightarrow 0
    \qquad
    \text{as }D\to\infty .
\]
Thus, the exponential taper eliminates any finite-order far-field tail
rather than merely bounding it.
\end{proposition}

(Proof: Appendix~\ref{app:exp_vanishing_influence}; the exponential form is
a convenient envelope, not the unique choice.)

The control conditions form a guarantee hierarchy. L0,
\(\omega(D)\to0\), gives eventual inward Gaussian radial drift for each
fixed finite \(\rho\); L1, \(\omega(D)\sqrt D\to0\), removes the
absolute negative far-field force; L2, \(\omega(D)D=O(1)\), bounds the
Gaussian reuse loop gain; and L3,
\(D^k\omega(D)\to0\) for every finite \(k\), removes all finite-order
tails. Global scaling fails L0, reciprocal-quadratic tapering is critical at
L2, and exponential tapering satisfies L0--L3. 

This hierarchy also organizes the experimental controls. Positive-only and
global scaling ignore remoteness; reciprocal tapers have polynomial tails, and the quadratic member meets the matching
order; DRPO uses the exponential member that satisfies
L0--L3. These controls share the actor form in
Equation~\eqref{eq:drpo_actor_field} and differ only in their negative-branch
coefficient \citep{peng2019advantage,nair2020awac,kostrikov2021offline,
grigsby2021closer,arnal2025asymmetric,leroux2025tapered}.

\subsection{Output-Coordinate Closed-Loop Guarantees}

We now record the output-coordinate consequences of the DRPO update once
positive attraction is included.

\begin{theorem}[Gaussian DRPO boundedness]
\label{thm:gaussian_drpo_boundedness}
Let \(p>0\), \(\Sigma\succ0\), and
\[
F_\omega(\mu)=p\Sigma^{-1}(m_+-\mu)
+\sum_iq_i\omega_i(D_\mu(a_i))\Sigma^{-1}(\mu-a_i).
\]
For finitely many negative atoms with exponential tapers, define
\begin{equation}
\begin{aligned}
B_\omega&=\sup_\mu\left\|\sum_iq_i\omega_i(D_i)
\Sigma^{-1}(\mu-a_i)\right\|_2,\\
R^*&=\frac{B_\omega}{p\lambda_{\min}(\Sigma^{-1})}.
\end{aligned}
\label{eq:gaussian_bound_radius}
\end{equation}
Every continuous trajectory satisfies
\(\limsup_{t\to\infty}\|\mu(t)-m_+\|_2\le R^*\), and the field has at
least one finite equilibrium. If an uncontrolled stable equilibrium
places every negative atom strictly below its threshold, the controlled
and original fields have the same local linearization in a neighborhood of
that equilibrium.
\end{theorem}

Appendix~\ref{app:gaussian_drpo_boundedness} gives an explicit
branch-wise bound, including the case where the shifted threshold is
negative and tapering is active from radius zero. Along every escaping
sequence,
\[
\rho_{\mathrm{eff}}(\mu)
=
p^{-1}\sum_i q_i\omega_i(D_i)
\longrightarrow 0;
\]
this explains how the remote negative source disappears, but it is a
diagnostic consequence rather than the boundedness argument.

\begin{theorem}[Categorical self-limiting suppression]
\label{thm:categorical_self_limiting}
For a repeatedly suppressed action \(i\), define
\(y_i=z_i-\log\sum_{j\ne i}e^{z_j}\), so \(\pi_i=\sigma(y_i)\).
Under an uncontrolled negative atom of mass \(q>0\),
\[
\dot y_i=-q(1-\pi_i)c_K(t),
\qquad \frac K{K-1}\le c_K(t)\le2.
\]
Hence \(y_i\) eventually decreases linearly and \(\pi_i\) exponentially.
With \(\omega_i=\min\{1,(e^{\tau_i}\pi_i)^\beta\}\),
\(\beta=\lambda/c\), the far branch instead satisfies
\[
y_i(t)=-\beta^{-1}\log t+O(1),
\qquad \pi_i(t)=\Theta(t^{-1/\beta}).
\]
Thus tapering changes persistent suppression from exponential to
polynomial probability decay; it does not freeze the action probability.
\end{theorem}

The output-coordinate proof and finite-crossing conditions are given in
Appendices~\ref{app:categorical_self_limiting}
and~\ref{app:categorical_restoring_balance}.

DRPO's exponential taper also arises as the optimizer of a
generalized-I-divergence-constrained reweighting of the finite
negative-update measure; unlike standard DRO over normalized distributions,
this construction may reduce total negative-update mass
(Appendix~\ref{app:finite_measure_variational}).

\section{Experiments}
\label{sec:experiments}

The experiments proceed from occurrence to identification, control, and
transfer. We
first test whether the predicted far-field pattern appears in realistic
continuous-control and autoregressive policies. Controlled environments then
separate coefficient magnitude from policy geometry, trace the resulting
update into policy behavior, and compare negative-weighting rules under known
near/far structure. The final task-level sections retain the D4RL and Countdown
evaluation targets.

\subsection{Environments and Datasets}
\label{sec:environments_datasets}

We use two controlled environments and two external diagnostic domains.
C-U1 is a contextual bandit with a Gaussian policy over continuous actions,
while D-U1 uses unordered categorical actions with shared representations.
They provide known ground truth for source isolation, policy dynamics, and
held-out-context generalization. The nine D4RL locomotion dataset cells and
Countdown arithmetic generation provide continuous offline-control and
autoregressive sequence diagnostics, respectively. Full configurations are
in Appendix~\ref{app:experimental_details}.

\paragraph{Conventions.}
Controlled source-isolation experiments report policy-score response,
whereas neural probes use implemented actor-gradient magnitude only to
test whether the same pattern appears in trainable updates.
Task-performance collapse, policy-support or variance-boundary events,
and NaN/Inf failures are reported as distinct outcome classes;
protocol-specific definitions are in
Appendix~\ref{app:controlled_protocols}.

\subsection{External Diagnostic of Far-Field Repulsion}
\label{sec:external_far_field}

For D4RL, probing is performed around fixed Positive-only Gaussian reference
actors without policy updates; Countdown provides the autoregressive
counterpart. In both domains, learner-relative remoteness varies while
sample-quality magnitude is held fixed. Full protocols are given in
Appendices~\ref{app:d4rl_gradient_diagnostic}
and~\ref{app:countdown_methods}.

Across all nine D4RL locomotion dataset cells, the implemented
full-parameter actor-gradient diagnostic increases strongly with
policy-relative standardized distance, with per-dataset panels
(Figure~\ref{fig:app_d4rl9_gradient_panels}) showing that the pattern is not
driven by a single environment or dataset quality. The approximately linear
scaling agrees with the Gaussian prediction that mean-score magnitude grows
linearly with standardized distance. In Countdown, the fixed-coefficient
diagnostic remains substantial at high mean-token surprisal, consistent with
persistent categorical suppression. These probes show that the predicted
near-to-far pattern appears in realistic neural policies; the controlled
experiments below separate remoteness from coefficient magnitude and test
whether far-field updates drive instability.

\subsection{Controlled Identification of the Causal Mechanism}
\label{sec:controlled_identification}

The external diagnostics still entangle sample quality, remoteness, critic or
verifier coefficients, data coverage, and subsequent dynamics. Controlled
environments vary coefficient magnitude and remoteness independently and
intervene directly on near- and far-field negative updates. Full constructions
and protocols are provided in
Appendix~\ref{app:controlled_protocols}.

\begin{figure}[t]
    \centering
    \includegraphics[
        width=\columnwidth
    ]{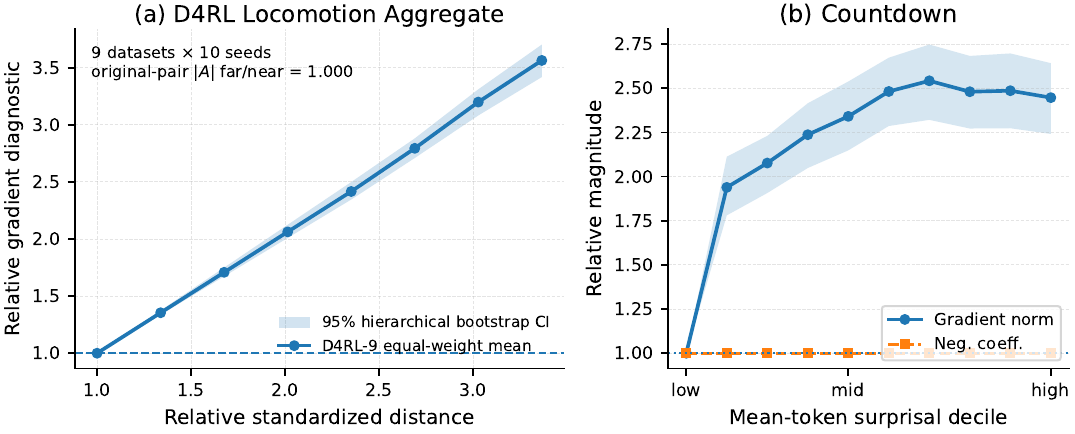}

    \caption{
        \textbf{External far-field diagnostics.}
        \textbf{(a)} D4RL-9 equal-weight aggregate of relative implemented
        actor-gradient magnitude for negative frozen-advantage transitions versus
        policy-relative standardized distance. Each far transition is normalized by
        its matched near partner within seed and dataset; the matched far/near
        absolute-advantage ratio checks matching quality (shading:
        hierarchical-bootstrap 95\% CI; per-dataset panels:
        Figure~\ref{fig:app_d4rl9_gradient_panels}).
        \textbf{(b)} Countdown relative implemented actor-gradient diagnostic and
        fixed negative coefficient versus mean-token surprisal decile among
        verifier-matched unsuccessful responses, normalized to the lowest-surprisal
        decile (shading: puzzle-level bootstrap 95\% CI).
    }
    \label{fig:external_far_field}
\end{figure}

\subsubsection{Remoteness as an Independent Source of Update Magnitude}
\label{sec:controlled_source}

\begin{figure}[t]
    \centering

    \includegraphics[
        width=0.92\columnwidth,
        trim={0pt 2pt 0pt 2pt},
        clip
    ]{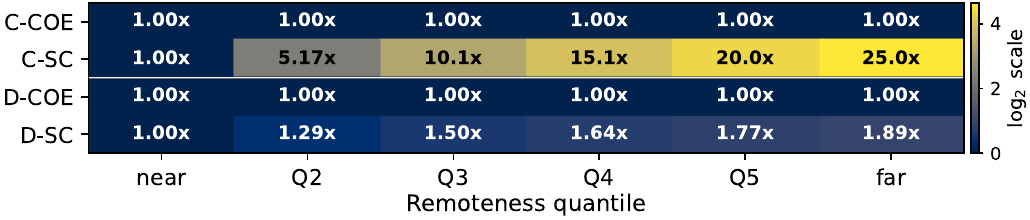}

    \vspace{0.35em}

    \includegraphics[
        width=0.92\columnwidth,
        trim={0pt 2pt 0pt 2pt},
        clip
    ]{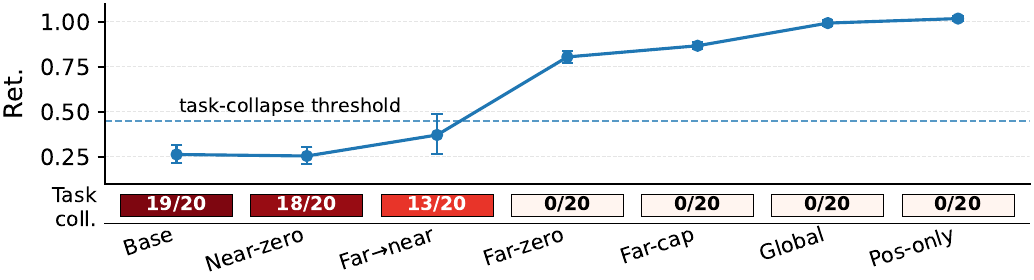}

    \vspace{-0.4em}

\caption{
    \textbf{Controlled source isolation and causal transmission.}
    \textbf{Top:} coefficient magnitude (COE) and policy-score response (SC)
    across remoteness bins, normalized to the near bin, for the C-U1 Gaussian
    protocol and the direct-softmax categorical reference.
    \textbf{Bottom:} C-U1 interventions; Ret. denotes reward retention
    (20-seed means with bootstrap CIs), and Task coll. denotes task-performance
    collapse counts.
}
    \label{fig:controlled_6_3}
\end{figure}

We first isolate whether remoteness changes update magnitude independently of
sample quality. In C-U1, negative actions are
placed on the same reward contour, giving them identical fixed
advantages but different standardized distances from the current
Gaussian policy. The categorical direct-softmax reference applies the
same factorization by fixing the negative coefficient while varying
current-policy surprisal.

The top panel of Figure~\ref{fig:controlled_6_3} separates the
coefficient term from the policy-score term. The coefficient magnitude
remains flat across remoteness bins, whereas the score response increases
with remoteness. Thus, the larger update cannot be explained
by farther samples having larger advantages; it is generated by the
policy-score geometry itself. The manifestation differs across action
spaces: Gaussian policies exhibit amplitude amplification, while
categorical policies approach a bounded direct-logit score but continue
to suppress the selected action toward the simplex boundary. The source-isolation experiment establishes where the update difference comes
from; the next intervention tests whether it propagates into drift and
collapse.

\subsubsection{Causal Transmission into Drift and Collapse}
\label{sec:controlled_transmission}

In the C-U1 intervention protocol, all branches share the
same data, fixed labels, initialization, optimizer, and training horizon,
and differ only in how near- and far-field negative updates are retained,
removed, capped, or globally rescaled. These matched-control invariants are
formalized in Appendix~\ref{app:controlled_shared}.

The bottom panel of Figure~\ref{fig:controlled_6_3} shows a clear rescue
pattern. Removing near-field negative updates leaves performance collapse
largely unchanged, whereas removing or capping far-field updates restores
reward retention and eliminates task-performance collapse. Global
rescaling also rescues the run, supporting the interpretation that
abnormal negative-update magnitude is a mediator. In contrast,
transferring the far-field budget to near-field negatives gives only
partial recovery, showing that the rescue is not obtained by simply
preserving the same total negative-update budget.

These interventions identify the far-field negative update, rather than
negative feedback in general, as a causal transmission path to policy
instability in the controlled Gaussian environment.

\subsection{From Useful Repulsion to DRPO Control}
\label{sec:repulsion_control}

The causal rescue does not imply that all negative feedback should be removed.
C-U1 and D-U1 place a hidden optimum beyond the positive-only target and make
local repulsion point toward it, while additional remote negatives create
far-field pressure. This construction separates useful repulsion from harmful
dominance and permits controlled comparisons of taper shape.

\subsubsection{Stable Extrapolation from Controlled Negative Feedback}
\label{sec:controlled_negative_strength}

\begin{figure}[t]
    \centering
    \includegraphics[
        width=0.95\columnwidth
    ]{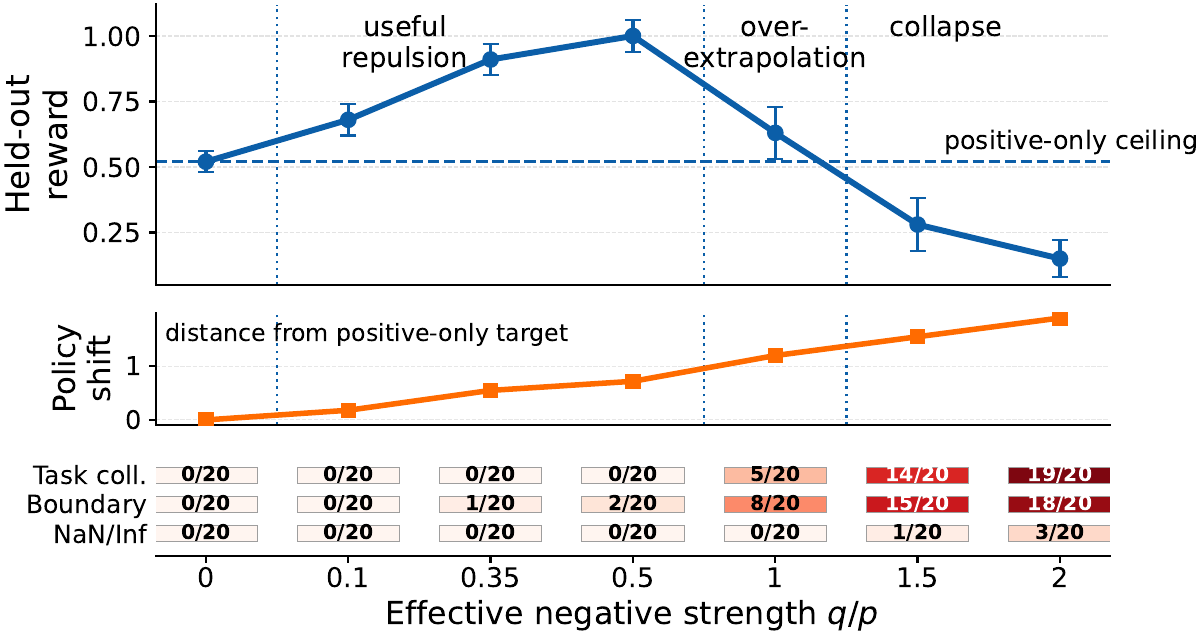}

    \vspace{-0.35em}
\caption{
    \textbf{Stable extrapolation from controlled negative feedback.}
    Increasing \(q/p\) first lifts held-out-context reward above the
    positive-only ceiling, then causes over-extrapolation and task-performance
    collapse; policy shift measures displacement from the positive-only target.
    The strip reports task collapse, boundary events, and NaN/Inf failures
    separately.
}
    \label{fig:phase_transition_6_4_1}
\end{figure}

We sweep the effective negative strength relative to positive attraction,
denoted by $q/p$, under the matched-control invariants above; protocol and
metric details are in Appendix~\ref{app:negative_strength_sweep}. The
positive-only setting is $q/p=0$.

Figure~\ref{fig:phase_transition_6_4_1} shows a non-monotone transition.
For small and moderate $q/p$, held-out-context reward rises above the
positive-only ceiling, showing that controlled negative feedback provides
useful repulsion beyond the positive-support solution. As $q/p$ grows
larger, policy shift continues to increase but reward declines,
indicating over-extrapolation rather than additional useful
generalization. At the largest strengths, task-performance collapse and
boundary events appear; numerical failures remain a separate outcome class.

Positive-only learning removes the harmful tail but also the beneficial
intermediate regime; remoteness-aware control should preserve local repulsion
while suppressing the far-field tail.

\begin{figure}[t]
\centering

\begin{minipage}[t]{0.68\columnwidth}
    \centering
    \vspace{0pt}
    \textbf{(a) Controlled taper comparison}\\[0.05em]
    \includegraphics[
        width=\linewidth
    ]{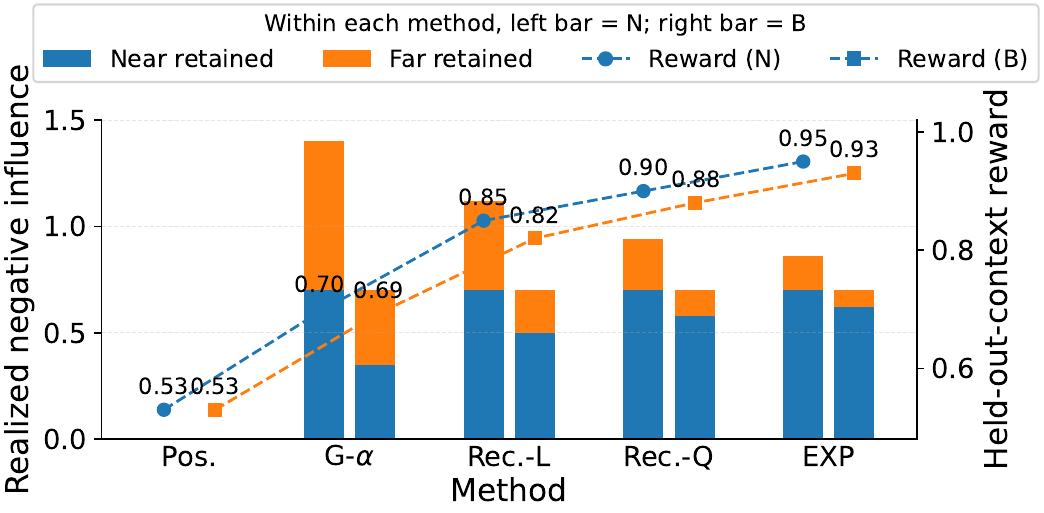}
\end{minipage}
\hfill
\begin{minipage}[t]{0.28\columnwidth}
    \centering
    \vspace{0pt}
    \textbf{(b) Transfer}\\[0.20em]
    \scriptsize
    \setlength{\tabcolsep}{2.2pt}
    \renewcommand{\arraystretch}{0.92}
    \resizebox{\linewidth}{!}{
    \begin{tabular}{lcc}
    \toprule
    Method & D4RL& Cnt. \\
    \midrule
    Pos.        & 70.1& 14.0 \\
    G-$\alpha$  & 72.9& 14.2\\
    Rec.-L      & 74.3& 13.6 \\
    Rec.-Q      & 72.1& 14.5 \\
    DRPO        & \textbf{77.7}& \textbf{15.7} \\
    \bottomrule
    \end{tabular}
    }

\end{minipage}

\vspace{-0.35em}
\caption{
\textbf{Controlled taper comparison and external transfer.}
\textbf{(a)} Matched operating points (left: near-field-retention matched;
right: aggregate-budget matched); bars decompose retained negative-update
magnitude into near- and far-field components, and dashed traces show
held-out-context reward.
\textbf{(b)} Countdown-0.5B fixed-budget transfer reports the best observed
two-seed late-window Pass@8 from each registered development sweep (full
curves: Figure~\ref{fig:app_countdown_taper_coefficient_response}); D4RL locomotion remains pending the formal run and terminal audit.
}
\label{fig:taper_control_and_transfer}
\end{figure}

\subsubsection{Controlled Comparison of Taper Shapes}
\label{sec:controlled_tapers}

We compare positive-only learning, global negative scaling,
reciprocal-linear tapering, reciprocal-quadratic tapering, and
exponential tapering while changing only the negative weighting rule. Two
matched comparisons are used. Near-field-retention matching preserves
comparable local negative-update magnitude and isolates differences in far-field
decay, whereas aggregate-budget matching tests whether a selective taper
outperforms a global reduction with the same total negative-update
budget.

Under both matched comparisons, exponential tapering retains near-field
repulsion while suppressing the far-field tail more rapidly than polynomial
tapers. Positive-only removes both components, whereas global scaling reduces
them indiscriminately. In the fixed-budget Countdown transfer summary, the
best observed point in the two-seed development sweep is DRPO's 15.7\%
late-window Pass@8, followed by Reciprocal-Quadratic (14.5\%), Positive-only
(14.0\%), Reciprocal-Linear (13.6\%), and Global-\(\alpha\) (14.2\%)
(full curves: Figure~\ref{fig:app_countdown_taper_coefficient_response}).

\subsubsection{External Validation of Remoteness Tapering}
\label{sec:external_tapers}

The same principle is tested externally: DRPO weights D4RL locomotion transitions by
standardized remoteness from the current Gaussian policy and Countdown
responses by mean-token surprisal, with protocol-matched methods sharing
data, labels, initialization, optimizer, and budget. Because external tasks
lack counterfactual near/far geometry, Figure~\ref{fig:taper_control_and_transfer}b
provides transfer evidence rather than a closed cross-task ranking; Hopper
entries remain open pending the formal run and terminal audit.

\subsection{Overall Task Performance}
\label{sec:overall_performance}

Having established the controlled mechanism, we evaluate DRPO as
a complete policy-optimization method on D4RL locomotion and Countdown.
We compare it with published domain-specific baselines and
protocol-matched controls spanning positive-only learning, global
negative scaling, and alternative remoteness tapers. All
protocol-matched methods use the same data and continuation budget. We
report normalized return for D4RL and verifier success for Countdown,
while task-performance failures, support or variance-boundary events,
and NaN/Inf numerical failures remain separate outcomes.

% Required packages:
% \usepackage{booktabs}
% \usepackage{graphicx}

\subsubsection{D4RL Locomotion}
\label{sec:d4rl_performance}

We evaluate the nine D4RL Gym locomotion tasks formed by HalfCheetah,
Hopper, and Walker2d under the medium, medium-replay, and medium-expert
datasets. These datasets vary in behavior quality and coverage, with
medium-replay containing the broadest mixture of historical policies.
Table~\ref{tab:d4rl_performance} compares DRPO with published offline-RL
baselines---CQL \citep{kumar2020conservative}, TD3+BC
\citep{fujimoto2021minimalist}, and IQL \citep{kostrikov2021offline}---%
under the D4RL evaluation protocol \citep{fu2020d4rl,todorov2012mujoco}. The shared objectives and
taper functions are defined in Appendix~\ref{app:shared_methods};
D4RL-specific score provenance, remoteness
construction, and implementation details are provided in
Appendix~\ref{app:d4rl_methods}.

Without adding an explicit behavior-policy constraint or a conservative
value penalty on top of the shared advantage-weighted actor form, DRPO
remains competitive on the D4RL-9 total; its only method-specific
component is a detached remoteness taper on the negative branch. This is
consistent with the far-field source-control route of
Section~\ref{sec:selective_tapering_bridge}, but should be read as
task-level transfer rather than a causal test of the mechanism.

\begin{table}[t]
\centering
\caption{Normalized returns on the D4RL-9 Gym locomotion benchmark.}
\label{tab:d4rl_performance}

\small
\setlength{\tabcolsep}{3.9pt}
\renewcommand{\arraystretch}{1.07}

\begin{tabular*}{\columnwidth}{@{\extracolsep{\fill}}lcccc}
\toprule
Dataset
&
CQL$^\dagger$
&
TD3+BC$^\dagger$
&
IQL$^\dagger$
&
\textbf{DRPO}
\\
\midrule

halfcheetah-medium
& 44.0 & \textbf{48.3} & \underline{47.4} & 47.3 \\

hopper-medium
& 58.5 & \underline{59.3} & \textbf{66.3} & 57.3 \\

walker2d-medium
& 72.5 & \underline{83.7} & 78.3 & \textbf{85.5} \\

\midrule

halfcheetah-replay
& \textbf{45.5} & \underline{44.6} & 44.2 & 43.3 \\

hopper-replay
& \textbf{95.0} & 60.9 & \underline{94.7} & 87.8 \\

walker2d-replay
& 77.2 & \textbf{81.8} & 73.9 & \underline{78.1} \\

\midrule

halfcheetah-expert
& \textbf{91.6} & 90.7 & 86.7 & \underline{91.1} \\

hopper-expert
& \textbf{105.4} & \underline{98.0} & 91.5 & \underline{98.0} \\

walker2d-expert
& 108.8 & \underline{110.1} & 109.6 & \textbf{110.4} \\

\midrule

\textit{D4RL-9 total}
& \textit{\underline{698.5}}
& \textit{677.4}
& \textit{692.4}
& \textbf{\textit{698.8}}
\\

\bottomrule
\end{tabular*}

\vspace{2pt}

\begin{minipage}{\columnwidth}
\small
\textit{Note.} ``replay'' and ``expert'' abbreviate ``medium-replay'' and
``medium-expert'', respectively. $\dagger$: published results rather than
protocol-matched reruns. Bold/underline: best/second-best among the methods
shown. DRPO reports the mean over 10 runs.
\end{minipage}

\end{table}

\subsubsection{Countdown}
\label{sec:countdown_performance}

We evaluate DRPO on held-out Countdown puzzles using
Qwen2.5-0.5B-Instruct~\citep{qwen2024qwen25}, where a response is
successful only if it forms a valid arithmetic expression that uses the
provided numbers and reaches the target value. Pass@8 is the primary task metric because it measures
sampling-level solution coverage under stochastic decoding, while
greedy verifier success provides a deterministic-decoding check.
Valid-expression rate is reported separately to distinguish reasoning
failure from formatting or execution failure.

We use a four-number Countdown variant \citep{hutton2002countdown}, following
the verifier-guided formulation of TinyZero \citep{pan2025tinyzero}. A
registered model-independent generator constructs one frozen prompt-indexed
bank of verified oracle solutions and verifier-incorrect completions, shared
with fixed labels and update budget by all protocol-matched methods;
method-specific learner-relative statistics are recomputed during training.
This zero-refresh limit isolates stale-negative reuse
(Appendices~\ref{app:countdown_dataset} and~\ref{app:countdown_methods}).

We compare DRPO with
AsymRE~\citep{arnal2025asymmetric}, the Joint Fitted-Reference
\(\beta\)-TOPR variant~\citep{leroux2025tapered}, and canonical
DPO~\citep{rafailov2023direct}, preceded by a short SFT stage to ensure
a task-capable initialization. The corresponding validation sweeps and
selected coefficients are reported in
Appendix~\ref{app:countdown_baseline_sensitivity}.
The shared objectives and taper functions are defined in
Appendix~\ref{app:shared_methods}.

\begin{table}[t]
\centering
\caption{Performance on held-out Countdown puzzles.}
\label{tab:countdown_performance}

\small
\renewcommand{\arraystretch}{1.10}
\setlength{\tabcolsep}{3.2pt}

\begin{tabular*}{\columnwidth}{@{\extracolsep{\fill}}lcccc}
\toprule
Metric
&
AsymRE
&
\makecell{\(\beta\)-TOPR\\(joint ref.)}
&
DPO\(^{\ddagger}\)
&
\textbf{DRPO}
\\
\midrule

Greedy success (\%)
& 6.80
& 7.53
& 2.32
& \textbf{7.62}
\\

Pass@8 (\%)
& 14.38
& 14.41
& 8.48
& \textbf{15.72}
\\

Valid expression (\%)
& \textbf{99.48}
& 99.14
& 66.24
& 99.40
\\

\bottomrule
\end{tabular*}

\vspace{2pt}

\begin{minipage}{\columnwidth}
\small
\textit{Note.}
AsymRE, \(\beta\)-TOPR, and DRPO initialize their trainable LoRA
adapters directly from the same Qwen2.5-0.5B-Instruct backbone;
\(^{\ddagger}\)DPO uses the same parameterization but initializes
its policy and frozen reference from the short SFT checkpoint
described in Appendix~\ref{app:countdown_methods}.
All entries are arithmetic means over evaluation steps 800, 900, 1000,
1100, and 1200 on the held-out validation split. The TOPR arm implements
the joint fitted-reference \(\beta\)-TOPR variant; it is not a canonical
frozen-behavior reproduction of \citet{leroux2025tapered}. Bold indicates
the highest observed value in each row.
\end{minipage}

\end{table}

DRPO delivers the strongest observed task performance while retaining a near-saturated valid-expression rate. Despite its task-capable warm start, DPO does not preserve its initial validation performance under repeated frozen-bank reuse in this setting.

\section{Conclusion}

Repeated off-policy reuse can transform useful negative feedback into a
self-reinforcing geometric effect. Our analysis connects the per-sample reuse
loop to aggregate regimes of stable displacement, drift, and instability.
DRPO interrupts that loop with a detached remoteness taper, yielding Gaussian
ultimate boundedness and polynomial rather than exponential categorical
suppression. Matched diagnostics and controlled interventions isolate policy
geometry as the source of the far-field update and show why selective control
can outperform both removing and globally scaling negative feedback. The main
open limitation is semantic: remoteness alone cannot distinguish obsolete
negative samples from remote samples that remain aligned with task improvement.
Combining geometric control with a reliable utility or alignment signal is a
natural next step.


\bibliography{example_paper,missing_references}
\bibliographystyle{icml2026}

%%%%%%%%%%%%%%%%%%%%%%%%%%%%%%%%%%%%%%%%%%%%%%%%%%%%%%%%%%%%%%%%%%%%%%%%%%%%%%%
%%%%%%%%%%%%%%%%%%%%%%%%%%%%%%%%%%%%%%%%%%%%%%%%%%%%%%%%%%%%%%%%%%%%%%%%%%%%%%%
% APPENDIX
%%%%%%%%%%%%%%%%%%%%%%%%%%%%%%%%%%%%%%%%%%%%%%%%%%%%%%%%%%%%%%%%%%%%%%%%%%%%%%%
%%%%%%%%%%%%%%%%%%%%%%%%%%%%%%%%%%%%%%%%%%%%%%%%%%%%%%%%%%%%%%%%%%%%%%%%%%%%%%%
\newpage
\appendix
\onecolumn
\section{Proofs for the Theoretical Results}
\label{app:theory_proofs}

\paragraph{Proof map.}
The score and reuse results are proved in
Sections~\ref{app:distance_strength}--\ref{app:gaussian_reuse_rate}; the
aggregate regimes and policy-family manifestations are proved in
Sections~\ref{app:aggregate_equilibrium} and
\ref{app:policy_family_manifestations}. The selective-tapering bridge,
reference restoration, Gaussian boundedness, control order, exponential tail,
and categorical results appear in Sections~\ref{app:far_field_bridge},
\ref{app:regularized_restoration}, \ref{app:gaussian_drpo_boundedness},
\ref{app:min_order_control}, \ref{app:exp_vanishing_influence}, and
\ref{app:categorical_self_limiting}--\ref{app:categorical_restoring_balance},
respectively.

\subsection{Proof of Proposition~\ref{prop:distance_strength}}
\label{app:distance_strength}

\begin{proof}
Let \(\delta=a-\mu\). For a Gaussian policy with fixed covariance
\(\Sigma\), the negative log-likelihood expands as
\[
    D_\mu(a)
    =
    C_\Sigma
    +
    \frac{1}{2}
    \delta^\top\Sigma^{-1}\delta,
\]
where
\(C_\Sigma=\frac{d}{2}\log(2\pi)+\frac{1}{2}\log|\Sigma|\)
is independent of \(a\) and \(\mu\). This proves
Equation~\eqref{eq:gaussian_remoteness_mahalanobis}.

The mean score function is
\begin{equation}
    \mathbf{s}_\mu(a)
    =
    \nabla_\mu\log\pi_\mu(a)
    =
    \Sigma^{-1}\delta,
\end{equation}
and hence
\begin{equation}
    R_\mu(a)
    =
    \left\|\mathbf{s}_\mu(a)\right\|_2^2
    =
    \delta^\top\Sigma^{-2}\delta.
    \label{eq:proof_gaussian_score_norm}
\end{equation}
Let \(x=\Sigma^{-1/2}\delta\). Then
\begin{equation}
    D_\mu(a)-C_\Sigma
    =
    \frac{1}{2}\|x\|_2^2,
    \qquad
    R_\mu(a)
    =
    x^\top\Sigma^{-1}x.
\end{equation}
Applying the Rayleigh quotient bounds gives
\begin{equation}
\begin{aligned}
    R_\mu(a)
    &\geq
    2\lambda_{\min}(\Sigma^{-1})
    \bigl(D_\mu(a)-C_\Sigma\bigr), \\
    R_\mu(a)
    &\leq
    2\lambda_{\max}(\Sigma^{-1})
    \bigl(D_\mu(a)-C_\Sigma\bigr).
\end{aligned}
\label{eq:proof_gaussian_score_remoteness_bounds}
\end{equation}
Thus, \(R_\mu(a)\) is linearly comparable to
\(D_\mu(a)-C_\Sigma\), with constants determined only by \(\Sigma\).
Since \(D_\mu(a)-C_\Sigma=\frac12 d_{\mathrm{Mah}}^2(a,\mu)\),
the Gaussian mean-score response grows without bound as the standardized
distance grows.

When \(\Sigma=\sigma^2I\), Equation~\eqref{eq:proof_gaussian_score_norm}
gives
\[
    R_\mu(a)
    =
    \frac{1}{\sigma^4}\|\delta\|_2^2,
    \qquad
    D_\mu(a)-C_\Sigma
    =
    \frac{1}{2\sigma^2}\|\delta\|_2^2.
\]
Therefore,
\[
    R_\mu(a)
    =
    \frac{2}{\sigma^2}
    \bigl(D_\mu(a)-C_\Sigma\bigr),
\]
which proves the isotropic exact relation stated in the proposition.

For the categorical policy, let
\begin{equation}
    p
    =
    \pi_z(a)
    =
    e^{-D_z(a)}.
\end{equation}
The score function with respect to the logits is
\begin{equation}
    \mathbf{s}_z(a)
    =
    \nabla_z\log\pi_z(a)
    =
    e_a-\pi_z.
\end{equation}
Its selected-action component is therefore
\begin{equation}
    \left|
        \frac{\partial}{\partial z_a}
        \log\pi_z(a)
    \right|
    =
    1-p
    =
    1-e^{-D_z(a)},
\end{equation}
which proves Equation~\eqref{eq:categorical_selected_score}. Since
\(1-e^{-D}\) is strictly increasing for \(D\geq0\) and converges to
\(1\) as \(D\rightarrow\infty\), the selected-logit response increases
with surprisal but remains bounded.

For completeness, the squared norm of the full categorical score function is
\begin{equation}
\begin{aligned}
    R_z(a)
    &=
    \left\|
        e_a-\pi_z
    \right\|_2^2 \\
    &=
    (1-p)^2
    +
    \sum_{j\neq a}\pi_z(j)^2.
\end{aligned}
\end{equation}
Because the non-selected probabilities sum to \(1-p\), the
Cauchy--Schwarz inequality gives
\begin{equation}
    \sum_{j\neq a}\pi_z(j)^2
    \geq
    \frac{(1-p)^2}{K-1}.
\end{equation}
The same sum is at most \((1-p)^2\). Hence,
\begin{equation}
\begin{aligned}
    \frac{K}{K-1}(1-p)^2
    &\leq
    R_z(a) \\
    &\leq
    2(1-p)^2.
\end{aligned}
\end{equation}
Substituting \(p=e^{-D_z(a)}\) shows that the full categorical logit-score
norm is bounded, completing the proof.
\end{proof}

\subsection{Proof of Theorem~\ref{thm:reuse_dynamics}}
\label{app:reuse_dynamics}

Before proving the generic reuse result, we verify the convexity condition
for the two policy-output coordinates used in the main text. For
fixed-covariance Gaussian means,
\[
D_\mu(a)=C_\Sigma+\frac12(a-\mu)^\top\Sigma^{-1}(a-\mu),
\qquad
\nabla_\mu^2D_\mu(a)=\Sigma^{-1}\succeq0.
\]
For categorical logits \(z\), with \(p=\mathrm{softmax}(z)\),
\[
D_z(a)=\log\sum_j e^{z_j}-z_a,
\qquad
\nabla_z^2D_z(a)=\operatorname{diag}(p)-pp^\top\succeq0,
\]
because
\[
v^\top(\operatorname{diag}(p)-pp^\top)v
=
\operatorname{Var}_{i\sim p}(v_i)\ge0.
\]

\begin{proof}
Let
\[
    f(u)=D(u),
    \qquad
    g_t=\nabla f(u_t),
    \qquad
    R_t=\|g_t\|_2^2.
\]

We first consider a negative coefficient
$\widehat{A}=-c$ with $c>0$. Defining
$\alpha=\eta c$, the update becomes
\[
    u_{t+1}
    =
    u_t+\alpha g_t.
\]
By convexity of $f$,
\begin{align}
    f(u_{t+1})
    &\geq
    f(u_t)
    +
    g_t^\top(u_{t+1}-u_t) \\
    &=
    f(u_t)+\alpha\|g_t\|_2^2.
\end{align}
Hence,
\[
    D_{t+1}
    \geq
    D_t+\eta|\widehat{A}|R_t.
\]

The gradient of a differentiable convex function is monotone, so
\[
    \bigl(g_{t+1}-g_t\bigr)^\top
    \bigl(u_{t+1}-u_t\bigr)
    \geq0.
\]
Using $u_{t+1}-u_t=\alpha g_t$ gives
\[
    g_{t+1}^\top g_t
    \geq
    \|g_t\|_2^2.
\]
By the Cauchy--Schwarz inequality,
\[
    \|g_{t+1}\|_2\|g_t\|_2
    \geq
    g_{t+1}^\top g_t
    \geq
    \|g_t\|_2^2.
\]
Whenever $\|g_t\|_2>0$, division by $\|g_t\|_2$ yields
\[
    \|g_{t+1}\|_2
    \geq
    \|g_t\|_2,
\]
and therefore $R_{t+1}\geq R_t$. The same result is immediate when
$R_t=0$. The inequality is strict whenever the gradient changes
nontrivially along a direction of positive curvature.

We next consider a positive coefficient
$\widehat{A}=c$ with $c>0$. Let
$\alpha=\eta c$, so that
\[
    u_{t+1}
    =
    u_t-\alpha g_t.
\]
Because $f$ has an $L$-Lipschitz gradient, the descent lemma gives
\begin{align}
    f(u_{t+1})
    &\leq
    f(u_t)
    -
    \alpha\|g_t\|_2^2
    +
    \frac{L\alpha^2}{2}\|g_t\|_2^2 \\
    &=
    f(u_t)
    -
    \alpha
    \left(
        1-\frac{L\alpha}{2}
    \right)
    \|g_t\|_2^2.
\end{align}
If $\alpha\leq1/L$, then
\[
    D_{t+1}
    \leq
    D_t-\frac{\alpha}{2}R_t
    =
    D_t-\frac{\eta\widehat{A}}{2}R_t.
\]

It remains to show that the score-function norm does not increase.
For a convex function with an $L$-Lipschitz gradient, the gradient is
$1/L$-cocoercive:
\[
    \bigl(g_t-g_{t+1}\bigr)^\top
    \bigl(u_t-u_{t+1}\bigr)
    \geq
    \frac{1}{L}
    \|g_t-g_{t+1}\|_2^2.
\]
Let $d=g_{t+1}-g_t$. Since
$u_t-u_{t+1}=\alpha g_t$, this implies
\[
    d^\top g_t
    \leq
    -
    \frac{1}{\alpha L}
    \|d\|_2^2.
\]
Consequently,
\begin{align}
    \|g_{t+1}\|_2^2
    &=
    \|g_t+d\|_2^2 \\
    &=
    \|g_t\|_2^2
    +
    2d^\top g_t
    +
    \|d\|_2^2 \\
    &\leq
    \|g_t\|_2^2
    +
    \left(
        1-\frac{2}{\alpha L}
    \right)
    \|d\|_2^2.
\end{align}
Because $\alpha L\leq1$, the final coefficient is nonpositive, and hence
\[
    R_{t+1}
    =
    \|g_{t+1}\|_2^2
    \leq
    \|g_t\|_2^2
    =
    R_t.
\]
This proves both claims.
\end{proof}

\subsection{Quantitative Gaussian Reuse Rates}
\label{app:gaussian_reuse_rate}

\begin{proposition}[Quantitative Gaussian reuse rates]
\label{prop:reuse_rate}
For an isotropic Gaussian mean policy with covariance \(\sigma^2I\),
one fixed negative sample of mass \(q>0\), and
\(X_t=D_t-C_\Sigma>0\), the uncontrolled update satisfies
\[
 X_{t+1}=\left(1+\eta q/\sigma^2\right)^2X_t,
\]
so remoteness grows geometrically. If the same sample is weighted by
\(\exp\{-a[D_t-\tau]_+\}\), \(a>0\), then its continuous-time
remoteness is \(O(a^{-1}\log t)\), and the discrete update obeys
\[
 X_t\le a^{-1}\{\log t+\log\log t+O(1)\}.
\]
This slowdown does not by itself prove aggregate boundedness; that requires
the closed-loop attraction established in
Theorem~\ref{thm:gaussian_drpo_boundedness}.
\end{proposition}

\begin{proof}[Proof of Proposition~\ref{prop:reuse_rate}]
Let \(\delta_t=\mu_t-a\) and
\(X_t=\|\delta_t\|_2^2/(2\sigma^2)=D_t-C_\Sigma\).
For an uncontrolled negative sample of mass \(q\),
\[
\delta_{t+1}
=\left(1+\frac{\eta q}{\sigma^2}\right)\delta_t,
\]
which gives the stated exact geometric recurrence.

Now let \(\omega_t=e^{-a[D_t-\tau]_+}\). Once
\(D_t>\tau\), write
\(b=\eta q/\sigma^2\) and
\(K_0=e^{a(\tau-C_\Sigma)}\). Then
\begin{equation}
X_{t+1}=X_t\left(1+bK_0e^{-aX_t}\right)^2.
\label{eq:app_tapered_gaussian_recurrence}
\end{equation}
Consequently \(\Delta X_t\ge0\), and for all sufficiently large
\(X_t\),
\begin{equation}
0\le\Delta X_t\le C X_te^{-aX_t}
\label{eq:app_gaussian_increment_bound}
\end{equation}
for a finite constant \(C\). Put \(Y_t=e^{aX_t}\). Since
\(aX_te^{-aX_t}\le1/e\),
Equation~\eqref{eq:app_gaussian_increment_bound} also gives a uniform
bound \(a\Delta X_t\le M\). Hence
\begin{align}
\Delta Y_t
&=Y_t(e^{a\Delta X_t}-1) \\
&\le C_1Y_t\,a\Delta X_t
\le C_2\log Y_t.
\label{eq:app_gaussian_y_increment}
\end{align}
The standard comparison sequence
\(\overline Y_t=K t\log t\), with \(K\) chosen larger than the
constant in Equation~\eqref{eq:app_gaussian_y_increment}, is eventually
a supersolution. Therefore
\(Y_t=O(t\log t)\) and
\[
X_t\le a^{-1}\{\log t+\log\log t+O(1)\}.
\]
If an application does not make the sign of \(\Delta X_t\) explicit,
negative increments already satisfy the desired upper comparison, while
the nonnegative branch obeys the argument above.

In continuous time, the far-branch equation is
\(\dot X=C_3Xe^{-aX}\). Thus
\(Y=e^{aX}\) satisfies \(\dot Y=C_3\log Y\), and the same comparison
gives \(X(t)=O(a^{-1}\log t)\).
\end{proof}

\subsection{Far-Field Radial Bridge}
\label{app:far_field_bridge}

\begin{proof}[Proofs of Lemma~\ref{lem:far_field_radial_balance} and
Corollary~\ref{cor:selective_vs_global}]
Set \(\mu=m_++rv\) with \(\|v\|_2=1\). Direct expansion gives
\begin{align}
\frac{\langle v,F(\mu)\rangle}{r}
={}&-p\,v^\top\Sigma^{-1}v
+\sum_iq_i v^\top\Sigma^{-1}v \\
&+\frac1r\sum_iq_i
\langle v,\Sigma^{-1}(m_+-a_i)\rangle.
\end{align}
The final term vanishes, proving the lemma.

With atom-specific weights, \(D_\mu(a_i)\to\infty\) along every ray.
For finitely many atoms, \(\omega_i(D_i)\to0\) therefore makes both the
weighted coefficient of \(r\) and the weighted constant term vanish.
The radial limit is
\(-p v^\top\Sigma^{-1}v\), which is strictly negative. This is a
pointwise-in-\(q/p\) statement: the radius at which it becomes negative
can grow with \(q/p\). Moreover,
\(\|\Sigma^{-1}(\mu-a_i)\|_2=O(\sqrt{D_i})\), so
\(\omega(D)\sqrt D\to0\) removes the absolute negative force. Finally,
for \(\omega_i\equiv\alpha\), the lemma applies with \(q\) replaced by
\(\alpha q\), yielding the condition \(\alpha q<p\).
\end{proof}

\subsection{Reference-Regularized Restoration}
\label{app:regularized_restoration}

\begin{proof}[Proof of Proposition~\ref{prop:regularized_restoration}]
The unweighted Gaussian field with a KL-geometry spring is
\[
F_\gamma(\mu)=\Sigma^{-1}
\{(q-p-\gamma)\mu+pm_+-qm_-+\gamma\mu_0\}.
\]
Whenever \(p-q+\gamma\ne0\), its unique equilibrium is
\[
\mu_\gamma^*
=\frac{pm_+-qm_-+\gamma\mu_0}{p-q+\gamma},
\]
and the field Jacobian is
\((q-p-\gamma)\Sigma^{-1}\). It is Hurwitz exactly when
\(\gamma>q-p\). This yields all three regimes in the proposition; for
\(p=q\), the displayed equilibrium has displacement proportional to
\(1/\gamma\). When \(p>q\), increasing \(\gamma\) contracts the
equilibrium toward \(\mu_0\), potentially sacrificing extrapolation.

For the Euclidean spring \(-\gamma(\mu-\mu_0)\), the Jacobian is
\((q-p)\Sigma^{-1}-\gamma I\). In the negative-dominant regime it is
Hurwitz exactly when
\(\gamma>(q-p)\lambda_{\max}(\Sigma^{-1})\); this condition is not
interchangeable with the KL-geometry coefficient above.
\end{proof}

\subsection{Gaussian Closed-Loop Boundedness}
\label{app:gaussian_drpo_boundedness}

\begin{proof}[Proof of Theorem~\ref{thm:gaussian_drpo_boundedness}]
Let \(e=\mu-m_+\) and
\[
N_\omega(\mu)=\sum_iq_i\omega_i(D_i)
\Sigma^{-1}(\mu-a_i).
\]
The finite-atom exponential-taper field is locally Lipschitz and has at
most linear growth, so its solutions exist for all forward times.
For \(V(e)=\frac12\|e\|_2^2\),
\begin{align}
\dot V
&=-p e^\top\Sigma^{-1}e+e^\top N_\omega(\mu) \\
&\le-p\lambda_{\min}(\Sigma^{-1})\|e\|_2^2
+B_\omega\|e\|_2.
\label{eq:app_gaussian_lyapunov_bound}
\end{align}
Thus \(\dot V<0\) outside every ball of radius
\(R^*+\varepsilon\). The standard ultimate-boundedness argument gives
\[
\limsup_{t\to\infty}\|\mu(t)-m_+\|_2\le R^*.
\]
On the boundary of any ball centered at \(m_+\) with radius greater
than \(R^*\), Equation~\eqref{eq:app_gaussian_lyapunov_bound} is
strictly negative. Continuity of the field and the Poincar\'e--Bohl
inward-pointing theorem therefore imply at least one zero inside the
ball. This is an existence result, not a uniqueness or global-convergence
claim.

It remains to verify \(B_\omega<\infty\) and the stated constant. Put
\(r_i=\|\Sigma^{-1/2}(\mu-a_i)\|_2\), so
\(D_i=C_\Sigma+r_i^2/2\), and
\[
\|\Sigma^{-1}(\mu-a_i)\|_2
\le\sqrt{\lambda_{\max}(\Sigma^{-1})}\,r_i.
\]
For \(a=\lambda/c\) and
\(\widetilde\tau_i=\tau_i-C_\Sigma\), maximize
\[
g_i(r)=r\exp\{-a[r^2/2-\widetilde\tau_i]_+\},\qquad r\ge0.
\]
If \(\widetilde\tau_i<0\), the taper is active at zero and the maximum
occurs at \(r=a^{-1/2}\), giving
\(a^{-1/2}e^{a\widetilde\tau_i-1/2}\). If
\(\widetilde\tau_i\ge0\), the untapered plateau ends at
\(r_\tau=\sqrt{2\widetilde\tau_i}\). When
\(a^{-1/2}\le r_\tau\), the maximum is \(r_\tau\); otherwise it is the
same interior-tail value. These are exactly the three displayed
branches for \(G_i\), and the triangle inequality proves the bound on
\(B_\omega\).

For near-field fidelity, let \(\mu^*\) be an uncontrolled equilibrium
with \(D_{\mu^*}(a_i)<\tau_i\) for every negative atom. Strictness and
continuity provide a neighborhood in which every \(\omega_i=1\).
The controlled and uncontrolled fields, including their Jacobians, are
therefore identical there.

Finally, along any sequence \(\|\mu_n\|_2\to\infty\), each fixed atom
satisfies \(D_{\mu_n}(a_i)\to\infty\), hence
\(\omega_i(D_{\mu_n}(a_i))\to0\). Finiteness of the atom set gives
\(\rho_{\mathrm{eff}}(\mu_n)\to0\).
\end{proof}

\subsection{Derivation and Proof of
Theorem~\ref{thm:aggregate_equilibria}}
\label{app:aggregate_equilibrium}

We use a regular minimal exponential-family representation to cover the
continuous-action Gaussian and discrete-action categorical policies in a
common set of output coordinates. At a fixed context, write
\begin{equation}
    \pi_\xi(a)
    =
    h(a)
    \exp\left\{
        \xi^\top T(a)-\psi(\xi)
    \right\},
    \qquad
    \xi\in\mathcal H,
    \label{eq:app_exponential_family_policy}
\end{equation}
where \(\xi\) is the natural parameter, \(T(a)\) is the sufficient
statistic, and \(\psi\) is the log-partition function. Let
\begin{equation}
    \mathcal M
    =
    \left\{
        \nabla\psi(\xi):
        \xi\in\mathcal H
    \right\}
    \label{eq:app_mean_parameter_space}
\end{equation}
denote the attainable mean-parameter space.

For a fixed-covariance Gaussian mean policy,
\(\xi=\Sigma^{-1}\mu\), \(T(a)=a\), and
\(\nabla\psi(\xi)=\mu\). For a categorical policy, fixing a
reference-logit gauge gives a minimal \((K-1)\)-dimensional
representation in which \(T(a)\) is the corresponding reduced action
indicator and \(\nabla\psi(\xi)\) is the reduced probability vector.
Thus, the two cases differ in their sufficient statistics and attainable
mean spaces, but share the derivation below.

Write
\[
    \widehat A
    =
    \widehat A^+
    -
    \widehat A^-,
    \qquad
    \widehat A^+
    =
    \max\{\widehat A,0\},
    \qquad
    \widehat A^-
    =
    \max\{-\widehat A,0\}.
\]
Define the aggregate positive and negative update masses and their
unnormalized moments:
\begin{equation}
\begin{aligned}
    p
    &=
    \mathbb E_\nu[\widehat A^+(a)],
    \qquad
    q
    =
    \mathbb E_\nu[\widehat A^-(a)],
    \\
    \bar{\mathbf m}_+
    &=
    \mathbb E_\nu
    \left[
        \widehat A^+(a)T(a)
    \right],
    \\
    \bar{\mathbf m}_-
    &=
    \mathbb E_\nu
    \left[
        \widehat A^-(a)T(a)
    \right].
\end{aligned}
\label{eq:app_aggregate_statistics}
\end{equation}
When \(p>0\), define
\begin{equation}
    \mathbf m_+
    =
    \frac{\bar{\mathbf m}_+}{p}.
    \label{eq:app_positive_moment}
\end{equation}
This is the positive-only target used in the main text: it is the
mean-parameter point selected by positive feedback alone, provided that
it lies in \(\mathcal M\). When \(q>0\), define
\begin{equation}
    \mathbf m_-
    =
    \frac{\bar{\mathbf m}_-}{q}.
    \label{eq:app_negative_moment}
\end{equation}
This is the negative-feedback moment, namely the advantage-weighted
center of the negative feedback. It is not an attraction target; the
negative actor contribution repels the learner from this moment. The
unnormalized quantities
\(\bar{\mathbf m}_+\) and \(\bar{\mathbf m}_-\) remain well defined in
the boundary cases \(p=0\) or \(q=0\).

The signed population objective is
\begin{equation}
    \mathcal L(\xi)
    =
    \mathbb E_\nu
    \left[
        \widehat A(a)
        \log\pi_\xi(a)
    \right].
    \label{eq:app_signed_population_objective}
\end{equation}
Using Equation~\eqref{eq:app_exponential_family_policy} and
\(\widehat A=\widehat A^+-\widehat A^-\), we obtain
\begin{equation}
    \mathcal L(\xi)
    =
    \left(
        \bar{\mathbf m}_+
        -
        \bar{\mathbf m}_-
    \right)^\top\xi
    -
    (p-q)\psi(\xi)
    +
    C,
    \label{eq:app_signed_objective}
\end{equation}
where
\[
    C
    =
    \mathbb E_\nu
    \left[
        \widehat A(a)\log h(a)
    \right]
\]
is independent of \(\xi\). Therefore, the aggregate actor field is
\begin{equation}
    \mathbf F(\xi)
    =
    \nabla\mathcal L(\xi)
    =
    \bar{\mathbf m}_+
    -
    \bar{\mathbf m}_-
    -
    (p-q)\nabla\psi(\xi),
    \label{eq:app_aggregate_field}
\end{equation}
and
\begin{equation}
    \nabla^2\mathcal L(\xi)
    =
    -(p-q)\nabla^2\psi(\xi).
    \label{eq:app_signed_hessian}
\end{equation}
Regularity and minimality imply that
\(\nabla^2\psi(\xi)\succ0\) throughout \(\mathcal H\), and that
\(\nabla\psi\) is one-to-one from \(\mathcal H\) onto
\(\mathcal M\).

We analyze the continuous dynamics and discrete update
\begin{equation}
    \dot\xi
    =
    \mathbf F(\xi),
    \qquad
    \xi_{t+1}
    =
    \xi_t
    +
    \alpha\mathbf F(\xi_t),
    \qquad
    \alpha>0.
    \label{eq:app_aggregate_dynamics}
\end{equation}

\begin{proof}

\medskip
\noindent\textbf{Positive-only case.}
Suppose \(p>0\) and \(q=0\). Then
\(\bar{\mathbf m}_-=\mathbf 0\), and
Equation~\eqref{eq:app_aggregate_field} becomes
\begin{equation}
    \mathbf F(\xi)
    =
    p
    \left(
        \mathbf m_+
        -
        \nabla\psi(\xi)
    \right).
    \label{eq:app_positive_only_field}
\end{equation}
A finite equilibrium must therefore satisfy
\begin{equation}
    \nabla\psi(\xi^\star)
    =
    \mathbf m_+.
    \label{eq:app_positive_only_equilibrium}
\end{equation}
If \(\mathbf m_+\in\mathcal M\), existence follows because
\(\nabla\psi\) maps \(\mathcal H\) onto \(\mathcal M\), and uniqueness
follows because this map is one-to-one.

The continuous field Jacobian is
\begin{equation}
    J_{\mathbf F}(\xi^\star)
    =
    -p\nabla^2\psi(\xi^\star)
    \prec0,
    \label{eq:app_positive_only_continuous_jacobian}
\end{equation}
so the equilibrium is locally asymptotically stable. For the discrete
update, the local Jacobian is
\begin{equation}
    J_{\mathrm{disc}}(\xi^\star)
    =
    I
    -
    \alpha p\nabla^2\psi(\xi^\star).
    \label{eq:app_positive_only_discrete_jacobian}
\end{equation}
Its eigenvalues have magnitude smaller than one whenever
\begin{equation}
    0
    <
    \alpha
    <
    \frac{
        2
    }{
        p\lambda_{\max}
        \left(
            \nabla^2\psi(\xi^\star)
        \right)
    }.
    \label{eq:app_positive_only_stepsize}
\end{equation}
Hence, the positive-only equilibrium is locally stable under
sufficiently small discrete steps. If
\(\mathbf m_+\notin\mathcal M\), no finite equilibrium exists.

\medskip
\noindent\textbf{Positive-dominant regime.}
Suppose \(p>q>0\). Since both normalized moments are defined, a finite
equilibrium must satisfy
\begin{equation}
    \mathbf F(\xi^\star)
    =
    \mathbf 0,
\end{equation}
or equivalently
\begin{equation}
\begin{aligned}
    \nabla\psi(\xi^\star)
    &=
    \frac{
        \bar{\mathbf m}_+
        -
        \bar{\mathbf m}_-
    }{
        p-q
    }
    \\
    &=
    \frac{
        p\mathbf m_+
        -
        q\mathbf m_-
    }{
        p-q
    }
    =
    \mathbf m^\star.
\end{aligned}
\label{eq:app_equilibrium_condition}
\end{equation}
If \(\mathbf m^\star\in\mathcal M\), existence and uniqueness again
follow from the bijectivity of
\(\nabla\psi:\mathcal H\rightarrow\mathcal M\).

Since \(p-q>0\), Equation~\eqref{eq:app_signed_hessian} gives
\begin{equation}
    \nabla^2\mathcal L(\xi)
    \prec0.
\end{equation}
Thus, \(\mathcal L\) is strictly concave, and
\(\xi^\star\) is its unique finite maximizer.

For the continuous dynamics, the field Jacobian at the equilibrium is
\begin{equation}
    J_{\mathbf F}(\xi^\star)
    =
    -(p-q)
    \nabla^2\psi(\xi^\star)
    \prec0.
    \label{eq:app_continuous_jacobian}
\end{equation}
All eigenvalues are strictly negative, so
\(\xi^\star\) is locally asymptotically stable.

For the discrete update, the local Jacobian is
\begin{equation}
    J_{\mathrm{disc}}(\xi^\star)
    =
    I
    -
    \alpha(p-q)
    \nabla^2\psi(\xi^\star).
    \label{eq:app_discrete_jacobian}
\end{equation}
Let \(\lambda_i>0\) denote the eigenvalues of
\(\nabla^2\psi(\xi^\star)\). The corresponding eigenvalues of
\(J_{\mathrm{disc}}(\xi^\star)\) are
\(1-\alpha(p-q)\lambda_i\). Their magnitudes are strictly smaller than
one whenever
\begin{equation}
    0
    <
    \alpha
    <
    \frac{
        2
    }{
        (p-q)
        \lambda_{\max}
        \left(
            \nabla^2\psi(\xi^\star)
        \right)
    }.
    \label{eq:app_discrete_stepsize}
\end{equation}
This proves local discrete-time stability for sufficiently small
\(\alpha\).

Define \(\rho=q/p\). Algebraically,
\begin{equation}
\begin{aligned}
    \mathbf m^\star
    &=
    \frac{
        p\mathbf m_+
        -
        q\mathbf m_-
    }{
        p-q
    }
    \\
    &=
    \mathbf m_+
    +
    \frac{q}{p-q}
    \left(
        \mathbf m_+
        -
        \mathbf m_-
    \right)
    \\
    &=
    \mathbf m_+
    +
    \frac{\rho}{1-\rho}
    \left(
        \mathbf m_+
        -
        \mathbf m_-
    \right).
\end{aligned}
\label{eq:app_stable_extrapolation_identity}
\end{equation}
This proves Equation~\eqref{eq:signed_target_ratio}. In particular,
holding \(\mathbf m_+\) and \(\mathbf m_-\) fixed, the displacement from
the positive-only target is
\begin{equation}
    \mathbf m^\star-\mathbf m_+
    =
    \frac{\rho}{1-\rho}
    \left(
        \mathbf m_+
        -
        \mathbf m_-
    \right),
    \label{eq:app_equilibrium_displacement}
\end{equation}
whose magnitude increases with \(\rho\in(0,1)\) whenever
\(\mathbf m_+\neq\mathbf m_-\).

If \(\mathbf m^\star\notin\mathcal M\), no finite
\(\xi^\star\in\mathcal H\) can satisfy
Equation~\eqref{eq:app_equilibrium_condition}, and hence no finite
equilibrium exists.

More generally, consider any sequence
\(\mathbf m^\star_n\in\mathcal M\) that leaves every compact subset of
\(\mathcal M\), either by approaching its boundary or by becoming
unbounded, and define
\begin{equation}
    \xi^\star_n
    =
    (\nabla\psi)^{-1}
    \left(
        \mathbf m^\star_n
    \right).
\end{equation}
Suppose, for contradiction, that
\(\{\xi^\star_n\}\) remains in a compact subset
\(K\subset\mathcal H\). By continuity,
\(\nabla\psi(K)\) is a compact subset of \(\mathcal M\), but
\(\mathbf m^\star_n=\nabla\psi(\xi^\star_n)\) would then remain inside
that compact set. This contradicts the assumed behavior of
\(\mathbf m^\star_n\). Hence, the corresponding natural parameters must
leave every compact subset of \(\mathcal H\).

\medskip
\noindent\textbf{Critical regime.}
Suppose \(p=q>0\). Equation~\eqref{eq:app_aggregate_field} reduces to
\begin{equation}
\begin{aligned}
    \mathbf F(\xi)
    &=
    \bar{\mathbf m}_+
    -
    \bar{\mathbf m}_-
    \\
    &=
    p
    \left(
        \mathbf m_+
        -
        \mathbf m_-
    \right),
\end{aligned}
\label{eq:app_critical_field}
\end{equation}
which is independent of \(\xi\).

If \(\mathbf m_+\neq\mathbf m_-\), the continuous trajectory is
\begin{equation}
    \xi(t)
    =
    \xi(0)
    +
    tp
    \left(
        \mathbf m_+
        -
        \mathbf m_-
    \right),
    \label{eq:app_critical_continuous}
\end{equation}
and the discrete trajectory is
\begin{equation}
    \xi_t
    =
    \xi_0
    +
    t\alpha p
    \left(
        \mathbf m_+
        -
        \mathbf m_-
    \right).
    \label{eq:app_critical_discrete}
\end{equation}
Both exhibit persistent linear drift and admit no finite equilibrium.

If \(\mathbf m_+=\mathbf m_-\), then
\(\mathbf F(\xi)=\mathbf 0\) for every \(\xi\). Every point is
stationary, but a perturbed trajectory remains at the perturbed point
rather than returning to the original one. Hence, no point is an
isolated locally asymptotically stable equilibrium.

\medskip
\noindent\textbf{Negative-dominant regime.}
Suppose \(p<q\). Equation~\eqref{eq:app_signed_hessian} gives
\begin{equation}
    \nabla^2\mathcal L(\xi)
    =
    (q-p)
    \nabla^2\psi(\xi)
    \succ0.
\end{equation}
Thus, \(\mathcal L\) is strictly convex.

If a finite stationary point \(\xi^\star\) exists, the continuous field
Jacobian is
\begin{equation}
    J_{\mathbf F}(\xi^\star)
    =
    (q-p)
    \nabla^2\psi(\xi^\star)
    \succ0.
    \label{eq:app_negative_continuous_jacobian}
\end{equation}
All eigenvalues are positive, so the stationary point is repelling and
therefore unstable.

For the discrete update, the local Jacobian is
\begin{equation}
    J_{\mathrm{disc}}(\xi^\star)
    =
    I
    +
    \alpha(q-p)
    \nabla^2\psi(\xi^\star).
    \label{eq:app_negative_discrete_jacobian}
\end{equation}
Its eigenvalues are
\[
    1+\alpha(q-p)\lambda_i
    >
    1
\]
for every \(\alpha>0\) and every \(\lambda_i>0\). Therefore, every finite
stationary point is unstable under the discrete dynamics as well. If no
stationary point exists, there is trivially no finite stable equilibrium.

The four cases establish
Theorem~\ref{thm:aggregate_equilibria}.
\end{proof}

\subsection{Proof of
Theorem~\ref{thm:policy_family_manifestations}}
\label{app:policy_family_manifestations}

\begin{proof}
Let
\begin{equation}
    r
    =
    q-p
    >
    0.
\end{equation}

\medskip
\noindent\textbf{Gaussian policy.}
For a Gaussian policy with fixed covariance,
\begin{equation}
    \nabla_{\mu}
    \log\pi_{\mu}(a)
    =
    \Sigma^{-1}(a-\mu).
    \label{eq:app_gaussian_mean_score}
\end{equation}
The aggregate mean field is therefore
\begin{equation}
\begin{aligned}
    \mathbf{F}_{\mu}(\mu)
    &=
    \mathbb{E}_{\nu}
    \left[
        \widehat{A}(a)
        \Sigma^{-1}(a-\mu)
    \right] \\
    &=
    \Sigma^{-1}
    \left[
        p\mu_{+}
        -
        q\mu_{-}
        -
        (p-q)\mu
    \right].
\end{aligned}
\label{eq:app_gaussian_aggregate_field}
\end{equation}
Because $p-q=-r$ and
\begin{equation}
    \mu^{\dagger}
    =
    \frac{
        q\mu_{-}
        -
        p\mu_{+}
    }{
        r
    },
\end{equation}
Equation~\eqref{eq:app_gaussian_aggregate_field} becomes
\begin{equation}
    \mathbf{F}_{\mu}(\mu)
    =
    r\Sigma^{-1}
    \left(
        \mu-\mu^{\dagger}
    \right).
\end{equation}

For the continuous dynamics, define
\begin{equation}
    \delta(t)
    =
    \mu(t)-\mu^{\dagger}.
\end{equation}
Then
\begin{equation}
    \dot{\delta}(t)
    =
    r\Sigma^{-1}\delta(t),
\end{equation}
and hence
\begin{equation}
    \delta(t)
    =
    \exp
    \left(
        r\Sigma^{-1}t
    \right)
    \delta(0).
    \label{eq:app_gaussian_continuous_solution}
\end{equation}
Since $\Sigma^{-1}\succ0$,
\begin{equation}
    \left\|
        \delta(t)
    \right\|_2
    \geq
    \exp
    \left(
        r\lambda_{\min}
        \left(
            \Sigma^{-1}
        \right)t
    \right)
    \left\|
        \delta(0)
    \right\|_2.
    \label{eq:app_gaussian_continuous_growth}
\end{equation}
Therefore, unless
$\delta(0)=0$, the distance from the stationary point grows
exponentially and
$\|\mu(t)\|_2\rightarrow\infty$.

For the discrete update,
\begin{equation}
    \mu_{t+1}
    =
    \mu_t
    +
    \alpha
    \mathbf{F}_{\mu}(\mu_t),
\end{equation}
so
\begin{equation}
    \delta_{t+1}
    =
    \left(
        I
        +
        \alpha r\Sigma^{-1}
    \right)
    \delta_t.
\end{equation}
Iterating gives
\begin{equation}
    \delta_t
    =
    \left(
        I
        +
        \alpha r\Sigma^{-1}
    \right)^t
    \delta_0.
    \label{eq:app_gaussian_discrete_solution}
\end{equation}
All eigenvalues of the update matrix are strictly larger than one.
Consequently,
\begin{equation}
    \left\|
        \delta_t
    \right\|_2
    \geq
    \left[
        1
        +
        \alpha r
        \lambda_{\min}
        \left(
            \Sigma^{-1}
        \right)
    \right]^t
    \left\|
        \delta_0
    \right\|_2.
    \label{eq:app_gaussian_discrete_growth}
\end{equation}
Thus, every nonstationary discrete trajectory also diverges.

For any fixed historical action $a$,
\begin{equation}
\begin{aligned}
    \left\|
        \nabla_{\mu}
        \log\pi_{\mu}(a)
    \right\|_2
    &=
    \left\|
        \Sigma^{-1}(a-\mu)
    \right\|_2 \\
    &\geq
    \lambda_{\min}
    \left(
        \Sigma^{-1}
    \right)
    \left\|
        a-\mu
    \right\|_2.
\end{aligned}
\label{eq:app_gaussian_score_lower_bound}
\end{equation}
Since $\|\mu\|_2\rightarrow\infty$ while $a$ is fixed,
$\|a-\mu\|_2\rightarrow\infty$.
Equation~\eqref{eq:app_gaussian_score_lower_bound} therefore proves
Equation~\eqref{eq:gaussian_unbounded_score}.

\medskip
\noindent\textbf{Categorical policy.}
Let there be $K\geq2$ actions and write
\begin{equation}
    \pi_z
    =
    \operatorname{softmax}(z).
\end{equation}
Because adding a constant to every logit does not change the policy, we
work on the gauge-fixed subspace
\begin{equation}
    \mathcal{H}_0
    =
    \left\{
        z\in\mathbb{R}^{K}
        :
        \mathbf{1}^{\top}z=0
    \right\}.
    \label{eq:app_categorical_gauge}
\end{equation}

Let $\rho_{+}$ and $\rho_{-}$ denote the normalized positive and negative
weighted action distributions, and define
\begin{equation}
    b
    =
    p\rho_{+}
    -
    q\rho_{-}.
\end{equation}
The signed categorical objective can be written as
\begin{equation}
    \mathcal{L}(z)
    =
    b^{\top}z
    +
    r
    \log
    \left(
        \sum_{j=1}^{K}
        e^{z_j}
    \right)
    +
    C,
    \label{eq:app_categorical_objective}
\end{equation}
because $p-q=-r$. Its gradient and Hessian are
\begin{equation}
    \nabla\mathcal{L}(z)
    =
    b+r\pi_z
\end{equation}
and
\begin{equation}
    \nabla^2\mathcal{L}(z)
    =
    r
    \left[
        \operatorname{Diag}(\pi_z)
        -
        \pi_z\pi_z^{\top}
    \right].
    \label{eq:app_categorical_hessian}
\end{equation}
For every nonzero
$v\in\mathcal{H}_0$,
\begin{equation}
\begin{aligned}
    v^{\top}
    \nabla^2\mathcal{L}(z)
    v
    &=
    r
    \operatorname{Var}_{A\sim\pi_z}
    \left[
        v_A
    \right] \\
    &>
    0,
\end{aligned}
\end{equation}
because every softmax probability is positive and a nonzero vector in
$\mathcal{H}_0$ cannot be constant. Hence,
$\mathcal{L}$ is strictly convex on $\mathcal{H}_0$ and has at most one
finite stationary point.

Consider first the continuous gradient flow
\begin{equation}
    \dot{z}
    =
    \nabla\mathcal{L}(z).
\end{equation}
Along every trajectory,
\begin{equation}
    \frac{d}{dt}
    \mathcal{L}(z(t))
    =
    \left\|
        \nabla\mathcal{L}(z(t))
    \right\|_2^2
    \geq
    0.
    \label{eq:app_categorical_objective_growth}
\end{equation}
Suppose a nonstationary trajectory were bounded. Its closure would be
compact, so $\mathcal{L}$ would remain bounded above. Integrating
Equation~\eqref{eq:app_categorical_objective_growth} would then give
\begin{equation}
    \int_0^\infty
    \left\|
        \nabla\mathcal{L}(z(t))
    \right\|_2^2
    dt
    <
    \infty.
\end{equation}
Since the gradient is Lipschitz on the compact trajectory closure, the
standard gradient-flow argument implies
\begin{equation}
    \nabla\mathcal{L}(z(t))
    \longrightarrow
    0.
\end{equation}
Any accumulation point would therefore be a stationary point.

If no finite stationary point exists, this is an immediate
contradiction. If a stationary point $z^{\dagger}$ exists, strict
convexity gives, for every $z\neq z^{\dagger}$,
\begin{equation}
    (z-z^{\dagger})^{\top}
    \left[
        \nabla\mathcal{L}(z)
        -
        \nabla\mathcal{L}(z^{\dagger})
    \right]
    >
    0.
\end{equation}
Since
$\nabla\mathcal{L}(z^{\dagger})=0$,
\begin{equation}
    \frac{d}{dt}
    \frac{1}{2}
    \left\|
        z(t)-z^{\dagger}
    \right\|_2^2
    >
    0
\end{equation}
along every nonstationary trajectory. Such a trajectory cannot have a
subsequence converging to $z^{\dagger}$. Hence, the only bounded
continuous trajectory is the stationary one itself.

For the discrete update
\begin{equation}
    z_{t+1}
    =
    z_t
    +
    \alpha
    \nabla\mathcal{L}(z_t),
    \qquad
    \alpha>0,
\end{equation}
convexity gives
\begin{equation}
\begin{aligned}
    \mathcal{L}(z_{t+1})
    &\geq
    \mathcal{L}(z_t)
    +
    \nabla\mathcal{L}(z_t)^{\top}
    (z_{t+1}-z_t) \\
    &=
    \mathcal{L}(z_t)
    +
    \alpha
    \left\|
        \nabla\mathcal{L}(z_t)
    \right\|_2^2.
\end{aligned}
\label{eq:app_categorical_discrete_growth}
\end{equation}
If the trajectory were bounded,
Equation~\eqref{eq:app_categorical_discrete_growth} would imply
\begin{equation}
    \nabla\mathcal{L}(z_t)
    \longrightarrow
    0.
\end{equation}
A convergent subsequence would therefore approach a finite stationary
point. If no such point exists, this is impossible. If the unique
stationary point $z^{\dagger}$ exists, then for
$z_t\neq z^{\dagger}$,
\begin{equation}
\begin{aligned}
    \left\|
        z_{t+1}-z^{\dagger}
    \right\|_2^2
    &=
    \left\|
        z_t-z^{\dagger}
    \right\|_2^2 \\
    &\quad
    +
    2\alpha
    (z_t-z^{\dagger})^{\top}
    \nabla\mathcal{L}(z_t) \\
    &\quad
    +
    \alpha^2
    \left\|
        \nabla\mathcal{L}(z_t)
    \right\|_2^2 \\
    &>
    \left\|
        z_t-z^{\dagger}
    \right\|_2^2.
\end{aligned}
\end{equation}
The trajectory therefore cannot have a subsequence converging to
$z^{\dagger}$. Thus, every nonstationary discrete trajectory leaves
every compact subset of $\mathcal{H}_0$.

To relate logit divergence to the probability boundary, note that on
$\mathcal{H}_0$ the inverse softmax representation is
\begin{equation}
    z_i
    =
    \log\pi_z(i)
    -
    \frac{1}{K}
    \sum_{j=1}^{K}
    \log\pi_z(j).
    \label{eq:app_inverse_softmax}
\end{equation}
If the probabilities remained in a compact subset of the simplex
interior, all coordinates would be bounded below by some
$\varepsilon>0$, and
Equation~\eqref{eq:app_inverse_softmax} would imply bounded logits.
Therefore, an unbounded gauge-fixed logit trajectory must contain a
subsequence along which at least one action probability approaches zero.

Finally, for every action $a$,
\begin{equation}
    \nabla_z\log\pi_z(a)
    =
    e_a-\pi_z.
\end{equation}
Its squared norm satisfies
\begin{equation}
\begin{aligned}
    \left\|
        e_a-\pi_z
    \right\|_2^2
    &=
    \left(
        1-\pi_z(a)
    \right)^2
    +
    \sum_{j\neq a}
    \pi_z(j)^2 \\
    &\leq
    \left(
        1-\pi_z(a)
    \right)^2
    +
    \left(
        \sum_{j\neq a}
        \pi_z(j)
    \right)^2 \\
    &=
    2
    \left(
        1-\pi_z(a)
    \right)^2 \\
    &\leq
    2.
\end{aligned}
\end{equation}
This proves the categorical score bound and completes the proof.
\end{proof}

\subsection{Proof of Minimum-Order Control}
\label{app:min_order_control}

For fixed-covariance Gaussian mean coordinates,
Proposition~\ref{prop:distance_strength} gives
\(R(D)=\Theta(D-C_\Sigma)=\Theta(D)\) in the far field. Thus there are
constants \(0<c_1\le c_2<\infty\) and \(D_0\) such that, for
\(D\ge D_0\),
\[
    c_1 D
    \le
    R(D)
    \le
    c_2 D .
\]
If \(\omega(D)R(D)\) is bounded, then for some \(M<\infty\),
\[
    \omega(D)R(D)\le M
    \qquad
    \text{for all sufficiently large } D .
\]
Using the lower bound gives
\[
    \omega(D)
    \le
    \frac{M}{c_1D}
    =
    O(D^{-1}).
\]
Conversely, \(\omega(D)\le CD^{-1}\) and the upper bound imply
\[
    \omega(D)R(D)
    \le
    CD^{-1}c_2D
    =
    O(1).
\]
The same comparison with little-\(o\) proves that the weighted loop gain
vanishes exactly when \(\omega(D)=o(D^{-1})\). Finally,
\(D-C_\Sigma=r^2/2\) converts \(D^{-1}\) to the equivalent
reciprocal-quadratic order \(r^{-2}\).

\subsection{Empirical Negative-Update Utility}
\label{app:negative_update_utility}

To operationalize the utility view used in
Section~\ref{subsec:distance_dependent_utility}, let
\(\widehat J_{\mathcal B}(\theta)\) denote an empirical estimate of task
performance on an evaluation batch \(\mathcal B\). For a negative sample
\(z=(s,a)\), define its raw repulsive update as
\begin{equation}
    g_\theta^-(z)
    =
    -\widehat A^-(z)
    \nabla_\theta \log \pi_\theta(a\mid s).
    \label{eq:raw_negative_update}
\end{equation}
Assigning it a weight \(w\in[0,1]\) produces the parameter step
\begin{equation}
    \Delta\theta(z;w)
    =
    \eta w g_\theta^-(z).
    \label{eq:weighted_negative_step}
\end{equation}
We define the empirical utility of this weighted negative update as
\begin{equation}
    \widehat{\mathcal U}_{\mathcal B}
    (z;w\mid\theta)
    =
    \widehat J_{\mathcal B}
    \!\left(
        \theta+\Delta\theta(z;w)
    \right)
    -
    \widehat J_{\mathcal B}(\theta).
    \label{eq:app_empirical_negative_update_utility}
\end{equation}
The corresponding utility-optimal weight is
\begin{equation}
    \widehat w_{\mathcal U}^{\star}(z\mid\theta)
    \in
    \arg\max_{w\in[0,1]}
    \widehat{\mathcal U}_{\mathcal B}(z;w\mid\theta).
    \label{eq:empirical_utility_optimal_weight}
\end{equation}

We distinguish this empirical utility from its population counterpart:
\begin{equation}
    \mathcal U_\theta(z;w)
    =
    J\!\left(
        \theta+\eta w g_\theta^-(z)
    \right)
    -
    J(\theta).
    \label{eq:population_negative_update_utility}
\end{equation}
If \(\widehat J_{\mathcal B}\) is an unbiased estimator of \(J\), then
\begin{equation}
    \mathbb E_{\mathcal B}
    \left[
        \widehat{\mathcal U}_{\mathcal B}(z;w\mid\theta)
    \right]
    =
    \mathcal U_\theta(z;w).
    \label{eq:empirical_utility_unbiased}
\end{equation}
Thus, empirical utility is a noisy, evaluation-dependent estimate rather
than an intrinsic deterministic label attached to the sample.

The following result gives sufficient conditions under which the
utility-optimal effective step vanishes in the far field.

\begin{lemma}[Far-field shrinkage of the utility-optimal step]
\label{lem:far_field_utility}
Let \(D\) index learner-relative remoteness, and let \(g(D)\neq0\)
denote the corresponding negative-update direction. Write
\begin{equation}
    I(D)
    =
    \|g(D)\|_2,
    \qquad
    v(D)
    =
    \frac{g(D)}{I(D)}.
    \label{eq:far_field_direction}
\end{equation}
Assume that the local update space admits an orthogonal decomposition
\begin{equation}
    \mathcal T
    =
    \mathcal T_{\mathrm{task}}
    \oplus
    \mathcal T_{\mathrm{far}},
    \label{eq:task_far_decomposition}
\end{equation}
and let \(P_{\mathrm{task}}\) denote the orthogonal projection onto
\(\mathcal T_{\mathrm{task}}\). Suppose that:
\begin{enumerate}
    \item \(\nabla J(\theta)\in\mathcal T_{\mathrm{task}}\);

    \item the normalized far-field update becomes asymptotically
    task-orthogonal,
    \begin{equation}
        \left\|
            P_{\mathrm{task}}v(D)
        \right\|_2
        \longrightarrow 0;
        \label{eq:task_projection_vanishes}
    \end{equation}

    \item there exist \(\rho_0>0\), \(\kappa>0\), and \(D_0\) such that,
    for every \(D\geq D_0\) and every \(\rho\in[0,\rho_0]\),
    \begin{equation}
        v(D)^\top
        \nabla^2J
        \!\left(
            \theta+\rho v(D)
        \right)
        v(D)
        \leq
        -\kappa.
        \label{eq:far_field_negative_curvature}
    \end{equation}
\end{enumerate}

Define the population utility of an effective step of length \(\rho\) by
\begin{equation}
    \mathcal V_D(\rho)
    =
    J\!\left(
        \theta+\rho v(D)
    \right)
    -
    J(\theta),
    \qquad
    0\leq\rho\leq\rho_0,
    \label{eq:effective_step_utility}
\end{equation}
and let
\begin{equation}
    \rho_{\mathcal U}^{\star}(D)
    \in
    \arg\max_{\rho\in[0,\rho_0]}
    \mathcal V_D(\rho).
    \label{eq:utility_optimal_effective_step}
\end{equation}
Then
\begin{equation}
    \rho_{\mathcal U}^{\star}(D)
    \longrightarrow 0
    \qquad
    \text{as }
    D\longrightarrow\infty.
    \label{eq:optimal_step_vanishes}
\end{equation}
Moreover, if
\[
    \langle\nabla J(\theta),v(D)\rangle<0,
\]
then every nonzero local step has negative utility:
\begin{equation}
    \mathcal V_D(\rho)<0
    \qquad
    \text{for all }
    \rho\in(0,\rho_0].
    \label{eq:negative_far_field_utility}
\end{equation}
\end{lemma}

\begin{proof}
Define the first-order task alignment
\begin{equation}
    a(D)
    =
    \left\langle
        \nabla J(\theta),
        v(D)
    \right\rangle.
    \label{eq:far_field_alignment}
\end{equation}
Because
\(\nabla J(\theta)\in\mathcal T_{\mathrm{task}}\), orthogonality gives
\begin{equation}
\begin{aligned}
    |a(D)|
    &=
    \left|
        \left\langle
            \nabla J(\theta),
            P_{\mathrm{task}}v(D)
        \right\rangle
    \right| \\
    &\leq
    \left\|
        \nabla J(\theta)
    \right\|_2
    \left\|
        P_{\mathrm{task}}v(D)
    \right\|_2.
\end{aligned}
\end{equation}
Equation~\eqref{eq:task_projection_vanishes} therefore implies
\begin{equation}
    a(D)\longrightarrow0.
    \label{eq:alignment_vanishes}
\end{equation}

For any \(\rho\in[0,\rho_0]\), Taylor's theorem gives some
\(\xi\in(0,\rho)\) such that
\begin{equation}
\begin{aligned}
    \mathcal V_D(\rho)
    &=
    \rho a(D) \\
    &\quad+
    \frac{\rho^2}{2}
    v(D)^\top
    \nabla^2J
    \!\left(
        \theta+\xi v(D)
    \right)
    v(D).
\end{aligned}
\end{equation}
Using Equation~\eqref{eq:far_field_negative_curvature},
\begin{equation}
    \mathcal V_D(\rho)
    \leq
    \rho a(D)
    -
    \frac{\kappa}{2}\rho^2.
    \label{eq:utility_quadratic_upper_bound}
\end{equation}

Since \(\mathcal V_D(0)=0\), any utility-maximizing step must attain
nonnegative utility. However, whenever
\begin{equation}
    \rho
    >
    \frac{
        2[a(D)]_+
    }{
        \kappa
    },
\end{equation}
the right-hand side of
Equation~\eqref{eq:utility_quadratic_upper_bound}
is strictly negative. Hence,
\begin{equation}
    0
    \leq
    \rho_{\mathcal U}^{\star}(D)
    \leq
    \frac{
        2[a(D)]_+
    }{
        \kappa
    }.
    \label{eq:optimal_step_upper_bound}
\end{equation}
Together with Equation~\eqref{eq:alignment_vanishes}, this proves
Equation~\eqref{eq:optimal_step_vanishes}.

If \(a(D)<0\), then for every \(\rho>0\),
\begin{equation}
    \rho a(D)
    -
    \frac{\kappa}{2}\rho^2
    <
    0.
\end{equation}
Equation~\eqref{eq:utility_quadratic_upper_bound} therefore implies
\(\mathcal V_D(\rho)<0\) for every nonzero local step, proving
Equation~\eqref{eq:negative_far_field_utility}.
\end{proof}

The lemma is a sufficient mechanism rather than a universal monotonicity
claim. It applies when far-field negative updates become increasingly
misaligned with the task-improving direction and the task objective is
locally harmed along that far-field direction. The curvature condition
can represent local losses caused by parameter drift or policy-support
contraction, provided those effects reduce the evaluated objective
\(J\). If far-field growth remains aligned with \(\nabla J(\theta)\), the
conclusion need not hold.

Finally, the effective step generated by a sample weight \(w\) is
\begin{equation}
    \rho(D;w)
    =
    \eta w I(D).
    \label{eq:weight_effective_step}
\end{equation}
Consequently, the population analogue of the utility-optimal sample
weight must satisfy
\begin{equation}
    \eta
    w_{\mathcal U}^{\star}(D)
    I(D)
    =
    \rho_{\mathcal U}^{\star}(D)
    \longrightarrow0.
    \label{eq:oracle_weighted_influence_vanishes}
\end{equation}
Thus, when the raw far-field response remains nonvanishing or grows, the
utility-optimal sample weight itself must shrink so that the weighted
response vanishes.

\subsection{Proof of Proposition~\ref{prop:exp_vanishing_influence}}
\label{app:exp_vanishing_influence}

\begin{proof}
For $D>\tau$,
\begin{equation}
    \omega_{\mathrm{Exp}}(D)
    =
    \exp
    \left[
        -\frac{\lambda}{c}(D-\tau)
    \right].
\end{equation}
Hence,
\begin{equation}
\begin{aligned}
    \omega_{\mathrm{Exp}}(D) I(D)
    &\leq
    C
    \exp
    \left[
        -\frac{\lambda}{c}(D-\tau)
    \right]
    (1+D)^m \\
    &\longrightarrow
    0,
\end{aligned}
\end{equation}
because exponential decay dominates every finite-order polynomial.

The corresponding weighted parameter step is
\begin{equation}
    \Delta\theta_{\mathrm{Exp}}(D)
    =
    \eta
    \omega_{\mathrm{Exp}}(D)
    g_\theta^-(D),
\end{equation}
and therefore
\begin{equation}
    \left\|
        \Delta\theta_{\mathrm{Exp}}(D)
    \right\|_2
    =
    \eta
    \omega_{\mathrm{Exp}}(D)
    I(D)
    \longrightarrow
    0.
\end{equation}

If $\widehat J_{\mathcal B}$ is locally $L_{\mathcal B}$-Lipschitz around
$\theta$, then for all sufficiently large $D$,
\begin{equation}
\begin{aligned}
    \left|
        \widehat{\mathcal U}_{\theta,\mathcal B}
        \bigl(z;\omega_{\mathrm{Exp}}(D)\bigr)
    \right|
    &\leq
    L_{\mathcal B}
    \left\|
        \Delta\theta_{\mathrm{Exp}}(D)
    \right\|_2 \\
    &\longrightarrow
    0.
\end{aligned}
\end{equation}
This proves both claims.
\end{proof}

\subsection{Categorical One-vs-Rest Dynamics}
\label{app:categorical_self_limiting}

\begin{proof}[Proof of Theorem~\ref{thm:categorical_self_limiting}]
For a single negative atom \(i\), the direct-logit field is
\[
\dot z=-q(e_i-\pi),
\qquad
\dot z_i=-q(1-\pi_i),
\qquad
\dot z_j=q\pi_j\quad(j\ne i).
\]
Define the normalized off-target probabilities
\(\bar\pi_j=\pi_j/(1-\pi_i)\). Differentiating
\(y_i=z_i-\log\sum_{j\ne i}e^{z_j}\) gives
\begin{align}
\dot y_i
&=-q(1-\pi_i)
-q\frac{\sum_{j\ne i}\pi_j^2}{1-\pi_i} \\
&=-q(1-\pi_i)c_K(t),
\qquad
c_K(t)=1+\sum_{j\ne i}\bar\pi_j^2.
\label{eq:app_categorical_exact_ode}
\end{align}
Because \(\bar\pi\) is a distribution on \(K-1\) actions,
\[
\frac1{K-1}\le\sum_{j\ne i}\bar\pi_j^2\le1,
\]
which proves the coefficient envelope without a concentration assumption.

We next justify entry into the far half-line before using its bounds.
The log-odds is strictly decreasing. Until \(\pi_i<1/2\), monotonicity
gives \(1-\pi_i(t)\ge1-\pi_i(0)>0\), so
Equation~\eqref{eq:app_categorical_exact_ode} has a strictly negative
upper bound. The trajectory therefore crosses \(\pi_i=1/2\) in finite
time. Thereafter \(1-\pi_i\ge1/2\), so \(y_i\to-\infty\) at a linear
rate and \(\pi_i=\sigma(y_i)\) decays exponentially.

For the generic coefficient, suppose \(j^*\ne i\) is the unique largest
off-target logit initially. For every \(k\ne i,j^*\),
\[
\frac{d}{dt}(z_{j^*}-z_k)=q(\pi_{j^*}-\pi_k)>0.
\]
The initial ordering is therefore preserved. If a gap stayed bounded,
then, once \(\pi_i\to0\), maximality of \(j^*\) and the positive initial
gap would give a positive lower bound on
\(\pi_{j^*}-\pi_k\), a contradiction. Hence all such gaps diverge,
off-target mass concentrates on \(j^*\), and \(c_K(t)\to2\). Under
complete off-target symmetry, symmetry is preserved and
\(\bar\pi_j=1/(K-1)\), giving
\(c_K(t)\equiv K/(K-1)\).

Now apply the taper
\(\omega_i=\min\{1,(e^{\tau_i}\pi_i)^\beta\}\). In the far branch,
let \(U=e^{-\beta y_i}\). Since
\(\pi_i e^{-y_i}=1-\pi_i\),
\begin{align}
\dot U
&=-\beta e^{-\beta y_i}\dot y_i \\
&=\beta q e^{\beta\tau_i}
c_K(t)(1-\pi_i)^{\beta+1}.
\label{eq:app_categorical_tapered_u}
\end{align}
After the finite crossing above, the right-hand side is bounded above
and below by positive constants. Thus \(U=\Theta(t)\), which is
equivalent to
\[
y_i(t)=-\beta^{-1}\log t+O(1),
\qquad
\pi_i(t)=\Theta(t^{-1/\beta}).
\]
\end{proof}

\subsection{Categorical Restoring-Force Balance}
\label{app:categorical_restoring_balance}

Write the one-dimensional restored dynamics as
\[
\dot y=F(y)=E(y)-S_{\mathrm{tap}}(y),
\]
where \(S_{\mathrm{tap}}>0\) is downward suppression and \(E>0\) is
upward restoration.

\begin{proposition}[Finite restoring crossing]
\label{prop:categorical_restoring_crossing}
Assume \(E\) and \(S_{\mathrm{tap}}\) are continuous and their zeros in
\([y_L,y_R]\) are isolated. If
\(F(y_L)>0\) and \(F(y_R)<0\), then the interval contains a zero at
which the phase line has a positive-to-negative crossing. Such a
crossing is locally asymptotically stable; if it is hyperbolic, its
condition is
\[
F'(y^*)<0
\quad\Longleftrightarrow\quad
[S_{\mathrm{tap}}-E]'(y^*)>0.
\]
\end{proposition}

\begin{proof}
The intermediate value theorem gives at least one zero. Because the
zeros are isolated on a compact interval, only finitely many occur; as
the endpoint signs differ, at least one has positive sign immediately
to its left and negative sign immediately to its right. The flow then
points toward that zero from both sides, proving phase-line stability. At a
hyperbolic crossing this sign change is equivalent to \(F'(y^*)<0\),
and the displayed derivative identity follows from \(F=E-S_{\mathrm{tap}}\).
\end{proof}

In the binary or symmetric one-vs-rest tail, the tapered suppression
has order \(e^{\beta y}\), whereas entropy restoration has order
\(|y|e^y\) as \(y\to-\infty\). Hence \(\beta\ge1\) supplies the
required far-tail ordering (with the \(|y|\) factor breaking the tie at
\(\beta=1\)), but it does not by itself verify the finite phase-line
conditions of Proposition~\ref{prop:categorical_restoring_crossing}.

\section{Distributional Reweighting Interpretations}
\label{app:distributional_reweighting}

\subsection{Finite-Measure I-Divergence Attenuation}
\label{app:finite_measure_variational}

DRPO's exponential attenuation admits an optimistic finite-measure
reweighting interpretation analogous to divergence-ball DRO, with the
constraint expressed in generalized I-divergence
\citep{csiszar1975divergence}. Optimistic or best-case distributional
optimization also has a formal precedent in likelihood approximation
and policy optimization
\citep{nguyen2019optimistic,song2020optimistic}. Our construction is
related but distinct: it reweights a finite measure of negative-update
mass according to learner-relative remoteness and may reduce that
measure's total mass. It is therefore an analogue of optimistic
divergence-ball reweighting, not a standard ambiguity set over normalized
probability distributions. The analogy is deliberately qualified: unlike
standard ambiguity sets over normalized probability measures, the present
construction permits attenuation of the total negative-update mass. Let
\(d\xi^-(z)=\widehat A^-(z)d\nu(z)\) and
\(h(z)=[D(z)-\tau(z)]_+\). Consider
\begin{equation}
\min_{0\le w\le1}\int wh\,d\xi^-
\quad\mathrm{s.t.}\quad
\int(w\log w-w+1)d\xi^-\le\delta.
\label{eq:finite_measure_variational}
\end{equation}

\begin{proposition}[Finite-measure exponential reweighting]
\label{prop:finite_measure_variational}
Let \(\delta_0=\xi^-(\{h>0\})\). For
\(0<\delta<\delta_0\), the constraint in
Equation~\eqref{eq:finite_measure_variational} is active and has a unique
finite \(\kappa>0\) with
\(w_\kappa(z)=e^{-h(z)/\kappa}\). At \(\delta=\delta_0\), the endpoint is
\(w_0=\mathbf1\{h=0\}\); for \(\delta>\delta_0\), zero objective is
already attainable and no finite one-to-one \(\kappa\)--\(\delta\)
relation remains. The later endpoint \(\delta\ge|\xi^-|\) merely makes
\(w\equiv0\) feasible.
\end{proposition}

\begin{proof}[Proof of Proposition~\ref{prop:finite_measure_variational}]
Use the convention
\(\phi(0)=1\) for
\(\phi(w)=w\log w-w+1\). For an active constraint, the Lagrangian with
multiplier \(\kappa>0\) is
\[
\mathcal L(w,\kappa)
=\int\{wh+\kappa\phi(w)\}\,d\xi^- -\kappa\delta.
\]
Pointwise stationarity on \(0<w\le1\) gives
\(h+\kappa\log w=0\), hence
\[
w_\kappa(z)=e^{-h(z)/\kappa}.
\]
For \(h=0\), this equals one; for \(h>0\), it lies strictly between
zero and one. Define
\(\mathcal D(\kappa)=\int\phi(w_\kappa)d\xi^-\). Dominated convergence gives
\[
\lim_{\kappa\to\infty}\mathcal D(\kappa)=0,
\qquad
\lim_{\kappa\downarrow0}\mathcal D(\kappa)
=\xi^-(\{h>0\})=\delta_0.
\]
On a set of positive \(h\), increasing \(\kappa\) strictly increases
\(w_\kappa\) toward one and strictly decreases \(\phi(w_\kappa)\).
Thus \(\mathcal D(\kappa)\) is continuous and strictly decreasing from
\(\delta_0\) to zero, proving existence and uniqueness for
\(0<\delta<\delta_0\).

At \(\delta=\delta_0\), the limit is
\(w_0=\mathbf1\{h=0\}\), which has zero objective and divergence
\(\delta_0\). For \(\delta>\delta_0\), the same zero-objective point is
strictly feasible, so the constraint need not be active and there is no
finite one-to-one multiplier relation. Finally,
\(\int\phi(0)d\xi^-=|\xi^-|\), so \(w\equiv0\) becomes feasible only at the later
endpoint \(\delta\ge|\xi^-|\).
\end{proof}

This problem derives learner-relative exponential attenuation. It is
distinct from quality-based hard selection: both are optimistic
reweighting views, but one controls remoteness under a finite negative
measure and the other selects a quality region. Neither problem implies
the other.

\subsection{Relation to Ratio- and Constraint-Based Controls}
\label{app:topr_relation}

\paragraph{Behavior-ratio identity.}
For total log probabilities, define
\[
r_{\mathrm{beh},\theta}(s,a)
=
\frac{\pi_\theta(a\mid s)}{\mu(a\mid s)},
\qquad
D_\theta(s,a)
=
-\log\pi_\theta(a\mid s),
\qquad
D_\mu(s,a)
=
-\log\mu(a\mid s),
\]
where \(\mu\) is the behavior policy. For any \(\beta>0\),
\[
\begin{aligned}
\min\left\{
    r_{\mathrm{beh},\theta}(s,a)^\beta,
    1
\right\}
&=
\min\left\{
    \left(
        \frac{\pi_\theta(a\mid s)}
             {\mu(a\mid s)}
    \right)^\beta,
    1
\right\} \\
&=
\exp\left\{
    -\beta
    \left[
        D_\theta(s,a)-D_\mu(s,a)
    \right]_+
\right\}.
\end{aligned}
\]
Thus, the clipped behavior ratio is exactly
Equation~\eqref{eq:drpo_exp_weight} with the sample-specific threshold
\(\tau=D_\mu(s,a)\) and rate \(\lambda/c=\beta\). At the atomic-action
level, canonical TOPR is the \(\beta=1\) case
\citep{leroux2025tapered}. The corresponding untruncated power weight
\(r_{\mathrm{beh},\theta}^{\beta}\) has the same far-field decay but may
exceed one when \(D_\theta<D_\mu\), whereas the thresholded taper remains
capped at one. Replacing both log probabilities by their
length-normalized values replaces
\(r_{\mathrm{beh},\theta}\) by
\(r_{\mathrm{beh},\theta}^{1/L}\).

Relative to this behavior-ratio form, DRPO treats the threshold and decay
rate as design quantities, recomputes remoteness under the current policy,
and applies the taper only to the negative branch.

\paragraph{Conditional-versus-population constants.}

For a conditionally reused far-branch sample, the coefficient contains
\(\pi/\mu\), so the conditional vector field retains a factor of
\(\mu^{-1}\). In the population field, however, the outer sampling mass
gives
\[
\mu\min\{\pi/\mu,1\}
=
\min\{\pi,\mu\},
\]
and cancels the ratio denominator on the far branch.

\paragraph{Experimental fitted-reference variant.}
The TOPR arm implements the joint fitted-reference
\(\beta\)-TOPR variant; it is not a canonical frozen-behavior
reproduction of \citet{leroux2025tapered}. The canonical population
cancellation above requires the same \(\mu\) to serve as both the outer
sampling mass and the ratio denominator. It therefore does not apply to
the experimental variant, whose outer measure is induced by
prompt-uniform replay with a per-prompt unique-negative mean and whose
denominator is \(\mu_{\mathrm{ref}}\). The experimental arm also uses
full-completion summed log probabilities rather than the
length-normalized ratio identity above.

For Gaussian atoms, the triangle inequality and the one-atom \(G_i\)
calculation in Appendix~\ref{app:gaussian_drpo_boundedness} first give
weighted constants. A \(\mu_{\min}\) expression follows only after
imposing finite support and replacing the weighted terms by their
worst-case bounds.

\paragraph{Refreshed proximal ratios.}
For a PPO-style proximal update, define
\[
r_{\mathrm{prox},\theta}(s,a)
=
\frac{\pi_\theta(a\mid s)}
     {\pi_{\mathrm{old}}(a\mid s)}.
\]
For the clipped surrogate
\citep{schulman2017proximal}, when \(\widehat A<0\), the effective
coefficient relative to
\(\widehat A\nabla_\theta\log\pi_\theta(a\mid s)\), away from the clipping
boundary, is
\[
r_{\mathrm{prox},\theta}
\mathbf 1
\left\{
    r_{\mathrm{prox},\theta}
    \geq
    1-\epsilon
\right\},
\]
rather than a pure indicator. The activation boundary is hard, but the
retained coefficient remains the proximal policy ratio. Refreshing
\(\pi_{\mathrm{old}}\) resets this ratio to one at the beginning of each
outer iteration. The resulting gate therefore constrains movement relative
to the latest anchor but does not directly encode the cumulative
learner-relative remoteness that increases under
Theorem~\ref{thm:reuse_dynamics}.

\paragraph{Movement, tethering, and source control.}
Reducing a fixed positive step size does not restore a stable finite
equilibrium under negative dominance
(Theorem~\ref{thm:aggregate_equilibria}). Reference regularization instead
restores an inward force by tethering the whole policy. When \(q>p\),
stability requires
\[
\gamma>q-p.
\]
When \(p>q\), strengthening the tether contracts the equilibrium toward
the reference
(Proposition~\ref{prop:regularized_restoration}). DRPO instead attenuates
the learner-relative remote negative source while leaving the positive
branch and near-field negative feedback unchanged. These mechanisms are
not mutually exclusive, but the dynamics of their combination lie outside
the fixed-data, fixed-advantage analysis considered here.

\section{Experimental Details}
\label{app:experimental_details}

\subsection{D4RL Locomotion Gradient-Diagnostic Protocol}
\label{app:d4rl_gradient_diagnostic}

\paragraph{Critic fitting and frozen advantages.}
For each of the nine D4RL locomotion dataset cells, we train one canonical
value critic for 100,000 optimizer updates. We then compute, standardize, and
freeze the transition-level learned-critic advantages. All actor seeds for a
given dataset use the same frozen critic outputs.

\paragraph{Positive-only reference actors.}
For seeds 100--109, we train a Gaussian reference actor for 100,000 optimizer
updates using only transitions with positive frozen advantage. The terminal
actor checkpoint is fixed before the gradient diagnostic and serves only to
define learner-relative distance and the implemented-gradient measurement.

\paragraph{Matched negative-transition probe.}
The probe population contains transitions with negative frozen advantage.
Near and far pools use standardized-distance quantiles 0.25 and 0.75. We match
absolute advantage using 20 bins and a 5\% relative tolerance, retain 256
matched pairs per seed, and evaluate the full actor-parameter gradient on 64
pairs per seed. Seven far-distance bins are reported. The probe performs no
optimizer update.

\paragraph{Aggregation and provenance.}
Every far point is normalized by its original matched near partner. Curves are
first aggregated within seed and dataset; the main panel then assigns equal
weight to the nine dataset cells on their shared distance support. Its 95\%
confidence interval uses 10,000 hierarchical bootstrap draws, resampling
dataset cells and then seeds, with bootstrap seed 20260723. The archived
point-level output records dataset, critic, actor, RunSpec, source, and
checkpoint provenance.

\begin{table*}[t]
\centering
\caption{Configuration of the D4RL locomotion gradient diagnostic.}
\label{tab:app_d4rl_gradient_protocol}
\small
\setlength{\tabcolsep}{5pt}
\begin{tabular}{ll}
\toprule
Component & Configuration \\
\midrule
Datasets & HalfCheetah, Hopper, Walker2d $\times$ medium, medium-replay, medium-expert \\
Critic & One learned value critic per dataset; 100,000 optimizer updates \\
Advantage & Transition-level learned-critic advantage; standardized and frozen \\
Reference actor & Positive-only Gaussian actor; 100,000 optimizer updates \\
Seeds & 100--109 \\
Probe population & Negative frozen-advantage transitions \\
Near/far pools & Standardized-distance quantiles 0.25 / 0.75 \\
Advantage matching & 20 bins; 5\% relative tolerance \\
Matched pool & 256 pairs per seed \\
Full-gradient subset & 64 pairs per seed \\
Distance summary & Seven far-distance bins with original-pair normalization \\
Probe updates & None \\
Cross-dataset summary & Within-seed/dataset normalization; equal dataset-cell weighting \\
Uncertainty & Hierarchical bootstrap over dataset cells and seeds \\
\bottomrule
\end{tabular}
\end{table*}

\begin{figure*}[t]
    \centering
    \includegraphics[width=\textwidth]{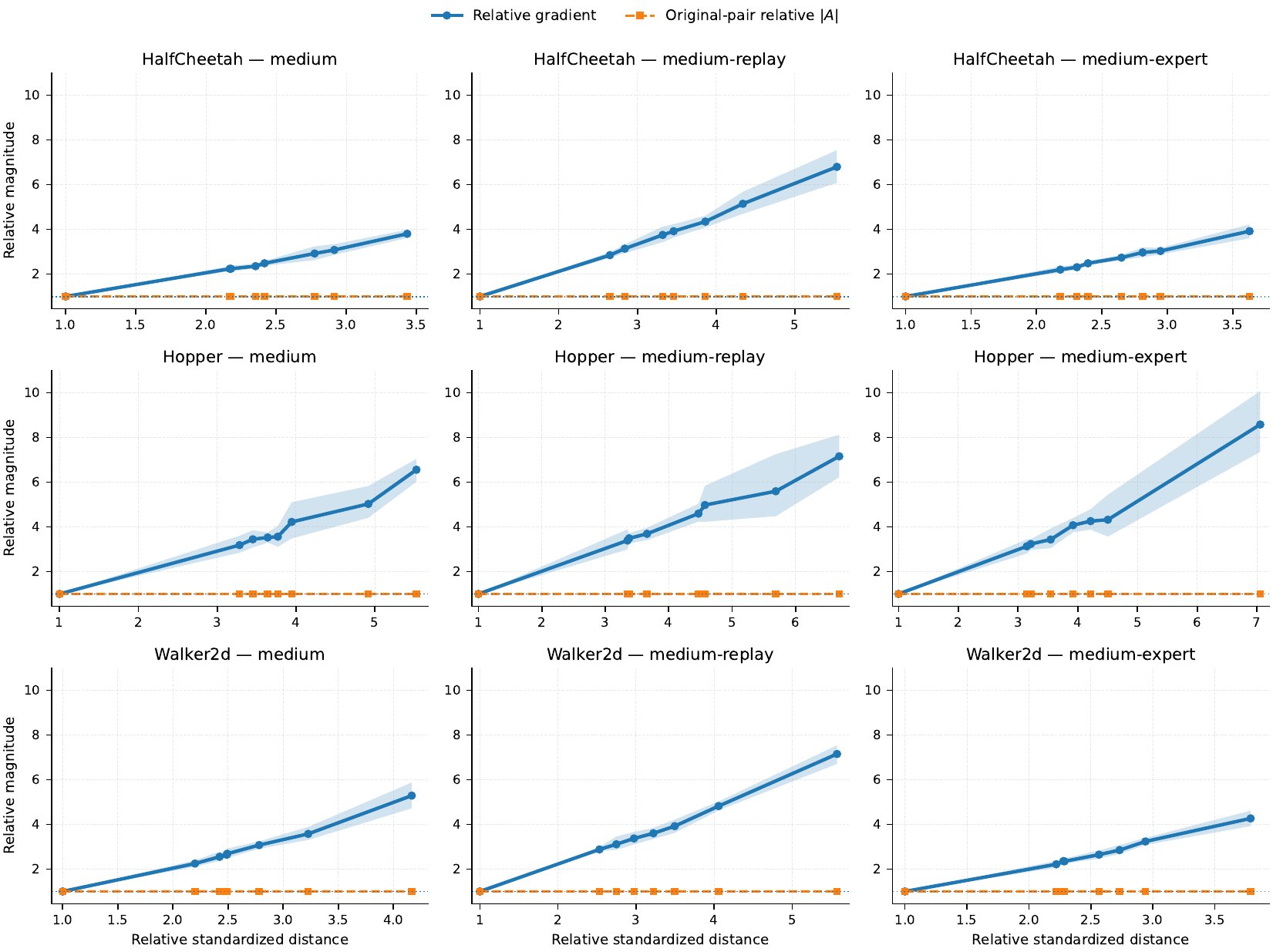}
    \caption{
    \textbf{Per-dataset D4RL locomotion gradient diagnostics.}
    Rows correspond to HalfCheetah, Hopper, and Walker2d; columns correspond to
    the medium, medium-replay, and medium-expert datasets. Each panel reports the
    relative implemented full-parameter actor-gradient diagnostic and matched
    absolute frozen advantage as functions of policy-relative standardized
    distance. Each far transition is normalized by its original matched near
    partner within seed; lines show 10-seed means and shaded regions show seed-level bootstrap 95\%
    confidence intervals. The common near-to-far increase across all nine dataset
    cells demonstrates that the aggregate result is not driven by a single task.
    }
    \label{fig:app_d4rl9_gradient_panels}
\end{figure*}

\subsection{Unified Continuous Environment C-U1}
\label{app:cu1_environment}

C-U1 is the shared controlled continuous environment used by all
continuous mechanism experiments. A context
$s\in\mathbb R^6$
is sampled from
$\mathcal N(0,I_6)$.
For each seed, we independently sample 4096 training contexts and 4096
test contexts from the same distribution. Test contexts never
participate in optimization and are used only to measure
held-out-context generalization.

\paragraph{State-dependent task geometry.}
Each context determines a two-dimensional positive-support center
$a_+(s)$ and a unit task direction $u(s)$. We define
\begin{equation}
\begin{aligned}
    [a_+(s)]_1
    &=
    0.70
    \tanh
    \bigl(
        0.85s_1
        -
        0.30s_2s_3
        +
        0.20\sin(1.6s_4)
    \bigr), \\
    [a_+(s)]_2
    &=
    0.65
    \tanh
    \bigl(
        -0.50s_2
        +
        0.35\cos(1.1s_5)
        +
        0.22s_1s_6
    \bigr).
\end{aligned}
\label{eq:cu1_positive_center}
\end{equation}
The direction angle is
\begin{equation}
\begin{aligned}
    \varphi(s)
    &=
    1.15
    \tanh
    \bigl(
        0.75s_1
        +
        0.50s_3
        -
        0.30s_6
    \bigr) \\
    &\quad+
    0.30\sin(1.35s_2),
\end{aligned}
\label{eq:cu1_direction_angle}
\end{equation}
with
\begin{equation}
    u(s)
    =
    \begin{bmatrix}
        \cos\varphi(s)\\
        \sin\varphi(s)
    \end{bmatrix},
    \qquad
    v(s)
    =
    \begin{bmatrix}
        -u_2(s)\\
        u_1(s)
    \end{bmatrix}.
    \label{eq:cu1_task_basis}
\end{equation}
The unobserved task optimum and the directionally useful negative anchor
are
\begin{equation}
\begin{aligned}
    a_\star(s)
    &=
    a_+(s)+0.70u(s), \\
    a_-(s)
    &=
    a_+(s)-0.50u(s).
\end{aligned}
\label{eq:cu1_task_points}
\end{equation}

\paragraph{Positive and negative actions.}
Each context is associated with four positive actions and eight negative
actions. The positive actions lie on a reward contour of radius
$r_+=0.75$
centered at $a_\star(s)$:
\begin{equation}
\begin{aligned}
    a_{\mathrm{pos},k}(s)
    =
    a_\star(s)
    +
    r_+
    \bigl(
        \cos\phi_k\,u(s)
        +
        \sin\phi_k\,v(s)
    \bigr),
\end{aligned}
\label{eq:cu1_positive_actions}
\end{equation}
where
\begin{equation}
\begin{aligned}
    \{\phi_k\}_{k=1}^{4}
    &=
    \{
        \pi-\theta_1,\,
        \pi+\theta_1,\,
        \pi-\theta_2,\,
        \pi+\theta_2
    \}, \\
    \theta_1
    &=
    0.20,
    \qquad
    \cos\theta_2
    =
    \frac{2(0.70)}{0.75}
    -
    \cos\theta_1.
\end{aligned}
\label{eq:cu1_positive_angles}
\end{equation}
This construction makes the four actions equal-reward while ensuring
that their centroid is exactly $a_+(s)$. It also leaves nonzero
conditional action spread, preventing deterministic maximum-likelihood
variance contraction from being confused with the far-field mechanism.

The eight negative actions lie on a second reward contour of radius
$r_-=1.20$:
\begin{equation}
\begin{aligned}
    a_{\mathrm{neg},j}(s)
    =
    a_\star(s)
    +
    r_-
    \bigl(
        \cos\psi_j\,u(s)
        +
        \sin\psi_j\,v(s)
    \bigr),
\end{aligned}
\label{eq:cu1_negative_actions}
\end{equation}
with
\begin{equation}
\begin{aligned}
    \{\psi_j\}_{j=1}^{8}
    =
    \left\{
        \pi,\,
        \frac{3\pi}{4},\,
        \frac{\pi}{2},\,
        \frac{\pi}{4},\,
        0,\,
        -\frac{\pi}{4},\,
        -\frac{\pi}{2},\,
        -\frac{3\pi}{4}
    \right\}.
\end{aligned}
\label{eq:cu1_negative_angles}
\end{equation}
The first negative action equals $a_-(s)$. All eight negative actions
have the same distance to $a_\star(s)$ and therefore the same reward and
fixed advantage, while having different distances from the current
policy.

\paragraph{Reward and fixed advantages.}
The ground-truth reward is
\begin{equation}
    R(s,a)
    =
    \exp
    \left(
        -
        \frac{
            \lVert a-a_\star(s)\rVert_2^2
        }{
            2(0.75)^2
        }
    \right).
    \label{eq:cu1_reward}
\end{equation}
Advantages are computed once using the fixed baseline $0.40$:
\begin{equation}
    \widehat A(s,a)
    =
    R(s,a)-0.40,
    \label{eq:cu1_advantage}
\end{equation}
and remain frozen throughout optimization. Consequently, all four
positive actions have positive advantages, all eight negative actions
have negative advantages, and negative-action quality is exactly matched
within each context.

\paragraph{Policy parameterization.}
The policy is an isotropic state-conditioned Gaussian,
\begin{equation}
    \pi_\theta(a\mid s)
    =
    \mathcal N
    \left(
        \mu_\theta(s),
        \sigma_\theta^2(s)I_2
    \right).
    \label{eq:cu1_policy}
\end{equation}
Both $\mu_\theta(s)$ and the scalar
$\log\sigma_\theta(s)$
are produced by a shared two-layer ReLU MLP with 64 hidden units per
layer and separate output heads. The initial standard deviation is
$0.60$, and the learnable-variance experiments use no artificial
variance clamp.  The continuous mechanism results reported in the main
text use the fixed-variance branch unless explicitly labeled otherwise. A finite crossing of
$\log\sigma_\theta(s)<-12$
is recorded as a support or variance-boundary event, whereas nonfinite
parameters or outputs are reported separately as NaN/Inf numerical
failure.

Training minibatches are sampled by context, and all actions associated
with a selected context are retrieved together. Thus, the replicated
actions within one context are not treated as independent contexts.
Positive and negative objectives are averaged separately before their
relative coefficient is applied, preventing the four-to-eight action
count ratio from mechanically changing the total negative-update mass.

\paragraph{Shared use across continuous protocols.}
All C-U1 experiments reuse the same contexts, reward geometry, fixed
advantages, and policy architecture. E1 compares equal-advantage
negative actions at different learner-relative distances. E2 applies
only positive updates and measures the same negative actions as phantom
probes. E3 dynamically partitions negative actions using the
standardized distance
\begin{equation}
    d_\theta(s,a)
    =
    \frac{
        \lVert a-\mu_\theta(s)\rVert_2
    }{
        \sigma_\theta(s)
    },
    \label{eq:cu1_standardized_distance}
\end{equation}
with the prespecified near--far threshold $d_\theta=5.0$, and applies
targeted near- and far-field interventions. E4 uses the aligned action
$a_-(s)$ as directionally useful local negative feedback and the
remaining contour actions as additional far-field pressure. The
experiment-specific coefficients, training horizons, seeds, taper
calibrations, and terminal criteria are reported separately with their
corresponding protocols.

\subsection{Controlled Categorical Environment D-U1}
\label{app:du1_environment}

D-U1 denotes the controlled categorical environment family used to
study persistent negative suppression, probability-support boundaries,
and shared-semantic generalization. Its E6 instantiation uses a fixed
catalogue of 64 unordered actions, each associated with a
four-dimensional unit semantic vector. Action identifiers are randomly
permuted for every seed, so numerical action indices carry no geometric
ordering.

For each seed, we independently sample 2048 training contexts and 2048
test contexts from
$s\sim\mathcal N(0,I_6)$.
The two splits follow the same context distribution, and test contexts
never participate in optimization. Results on this split are therefore
reported as held-out-context or unseen-state generalization rather than
OOD generalization.

\paragraph{Semantic action catalogue.}
Let
$\{e_i\}_{i=1}^{K}\subset\mathbb R^4$,
with $K=64$, denote the unit reward embeddings of the categorical
actions. They are sampled once per seed and randomly assigned to action
identifiers. For each context, two fixed linear maps generate a
demonstrated semantic target and a candidate improvement direction:
\begin{equation}
\begin{aligned}
    t_+(s)
    &=
    \operatorname{unit}
    \!\left(
        W_+^\top s
    \right), \\
    \widetilde u(s)
    &=
    W_u^\top s
    -
    \left\langle
        W_u^\top s,
        t_+(s)
    \right\rangle
    t_+(s), \\
    u(s)
    &=
    \operatorname{unit}
    \!\left(
        \widetilde u(s)
    \right).
\end{aligned}
\label{eq:du1_semantic_basis}
\end{equation}
The hidden optimal target and the local negative target are
\begin{equation}
\begin{aligned}
    t_\star(s)
    &=
    \operatorname{unit}
    \!\left(
        t_+(s)+\delta u(s)
    \right), \\
    t_-(s)
    &=
    \operatorname{unit}
    \!\left(
        t_+(s)-\delta u(s)
    \right),
\end{aligned}
\qquad
\delta=0.45.
\label{eq:du1_semantic_targets}
\end{equation}

\paragraph{Reward and action roles.}
The ground-truth reward of action $i$ is determined by its semantic
alignment with the hidden target:
\begin{equation}
    R(s,i)
    =
    \frac{1}{2}
    \left(
        1+
        t_\star(s)^\top e_i
    \right).
    \label{eq:du1_reward}
\end{equation}
The hidden optimal action is
\begin{equation}
    i_\star(s)
    =
    \arg\max_i
    t_\star(s)^\top e_i.
    \label{eq:du1_hidden_optimum}
\end{equation}
It is excluded from all positive demonstrations.

For each context, the four positive actions are the available actions
with the largest similarity to $t_+(s)$. After excluding the hidden
optimum and positive actions, the local negative action is the action
most aligned with $t_-(s)$. The four far-negative actions are selected
by their alignment with $-t_+(s)$. Thus,
\begin{equation}
\begin{aligned}
    \mathcal P(s)
    &=
    \operatorname{Top4}_{i\neq i_\star(s)}
    \,
    t_+(s)^\top e_i, \\
    i_{\mathrm{local}}(s)
    &=
    \arg\max_{i\notin
        \{i_\star(s)\}\cup\mathcal P(s)}
    t_-(s)^\top e_i, \\
    \mathcal F(s)
    &=
    \operatorname{Top4}_{i\notin
        \{i_\star(s),i_{\mathrm{local}}(s)\}
        \cup\mathcal P(s)}
    \,
    \bigl[-t_+(s)^\top e_i\bigr].
\end{aligned}
\label{eq:du1_action_roles}
\end{equation}
The positive and negative advantages are fixed throughout training:
\begin{equation}
    \widehat A(s,i)
    =
    \begin{cases}
        +1,
        & i\in\mathcal P(s),\\
        -1,
        & i=i_{\mathrm{local}}(s)
          \text{ or }i\in\mathcal F(s).
    \end{cases}
    \label{eq:du1_fixed_advantages}
\end{equation}
Consequently, differences between local and far-negative updates cannot
be attributed to different advantage magnitudes.

\paragraph{Shared-semantic categorical policy.}
The policy uses a two-layer tanh MLP with 64 hidden units per layer to
map the context to a unit semantic direction
$q_\theta(s)\in\mathbb R^4$.
Let $\widetilde e_i$ denote the action embedding used by the policy. The
categorical logits and probabilities are
\begin{equation}
\begin{aligned}
    \ell_{\theta,i}(s)
    &=
    \kappa_\theta(s)
    q_\theta(s)^\top
    \widetilde e_i, \\
    \pi_\theta(i\mid s)
    &=
    \frac{
        \exp\!\left(\ell_{\theta,i}(s)\right)
    }{
        \sum_{j=1}^{K}
        \exp\!\left(\ell_{\theta,j}(s)\right)
    }.
\end{aligned}
\label{eq:du1_policy}
\end{equation}
The concentration is either fixed at
$\kappa_\theta(s)=8$
or learned as
\begin{equation}
    \kappa_\theta(s)
    =
    \operatorname{softplus}
    \!\left(
        c_\theta(s)
    \right)
    +
    0.05,
    \label{eq:du1_concentration}
\end{equation}
with initial value 8 and no upper clamp.

In the semantically aligned condition,
$\widetilde e_i=e_i$.
In the shuffled control, the same policy embeddings are randomly
permuted across action identifiers while the reward embeddings and
training samples remain unchanged. This intervention preserves the
catalogue size and categorical parameterization but breaks the
correspondence between policy-side similarity and ground-truth semantic
utility.

\paragraph{Evaluation and failure events.}
We evaluate expected semantic reward, probability assigned to the hidden
optimal action, entropy, effective support, and normalized semantic
extrapolation on both training and held-out contexts. Task-performance
collapse, effective-support or concentration-boundary events, and
NaN/Inf numerical failure are reported separately. A policy can
therefore retain high reward while still approaching a support boundary;
the two outcomes are not treated as equivalent.

\paragraph{Shared use across categorical protocols.}
E6-A fixes the concentration and scans the strength of the local
negative update to identify the transition from the positive-only
ceiling to useful semantic extrapolation and then to performance
reversal. E6-B learns the concentration and applies near/far causal
interventions and matched controls. E6-C compares aligned and shuffled
policy embeddings to determine whether improvements on the hidden action
require the specified semantic correspondence.

E5 serves a distinct but complementary role. Its direct-softmax
diagnostic and categorical causal reconstruction isolate persistent
surprisal growth, logit-gap growth, and simplex-boundary suppression
under repeated negative updates. It shares the fixed-advantage and
near/far intervention logic of D-U1, but it does not use the 64-action
shared-semantic catalogue and does not support the E6 semantic
generalization claim.

\subsection{D4RL Locomotion Datasets}
\label{app:hopper_datasets}

The task-level evaluation uses HalfCheetah, Hopper, and Walker2d under
the medium, medium-replay, and medium-expert datasets
\citep{fu2020d4rl}; the mechanism diagnostics use Hopper, described
below.

Hopper is a continuous-control task with an
11-dimensional observation space and a three-dimensional bounded action
space. We use the D4RL
\texttt{hopper-medium},
\texttt{hopper-medium-replay}, and
\texttt{hopper-medium-expert}
datasets. The medium dataset contains trajectories collected from a
partially trained policy, medium-replay contains the corresponding replay
buffer collected throughout training, and medium-expert combines medium
and expert behavior. Hopper-medium-replay is used for the primary mechanism
diagnostic because its broader behavioral coverage provides a wider range of
learner-relative distances, while all three Hopper datasets are retained for
the task-level performance evaluation.

\subsection{Countdown Task and Puzzle Dataset}
\label{app:countdown_dataset}

Countdown is an arithmetic generation task in which the policy receives
a multiset of integers and a target value, and must generate an
arithmetic expression that evaluates exactly to the target. A valid
output must use every supplied number exactly once and may contain only
addition, subtraction, multiplication, division, and parentheses.

Puzzle instances and the frozen replay pool are constructed by our
registered model-independent generator; the cited works establish
the task provenance and prior LLM use, not the provenance of our
generated corpus.

Let
$x=(\mathcal N,T)$
denote a puzzle, where
$\mathcal N$
is the supplied multiset of numbers and
$T$
is the target. The policy action is a complete token sequence
$y$ representing the generated expression. Task success is determined
by a deterministic verifier:
\begin{equation}
    R(x,y)
    =
    \mathbf 1
    \left[
    \begin{aligned}
        &y
        \text{ is syntactically valid, uses }
        \mathcal N
        \text{ exactly once,} \\
        &\text{and evaluates to }T
    \end{aligned}
    \right].
    \label{eq:countdown_reward}
\end{equation}
Because multiple expressions can solve the same puzzle, evaluation is
based on verifier success rather than exact string matching.

The puzzle corpus is procedurally generated from valid arithmetic
expressions and divided into disjoint training, validation, and test
sets. Duplicate puzzles with the same number multiset and target are
removed across the corpus. Validation puzzles are used for checkpoint
selection, while the test split is reserved for final evaluation.

Countdown provides an external discrete-action setting with
autoregressive sequence generation and shared Transformer parameters.
For the diagnostic in Figure~\ref{fig:external_far_field}, unsuccessful
responses are verifier-matched and grouped by current mean-token surprisal.

\paragraph{Corpus construction and audit.}
The procedurally generated puzzle source contains 6{,}000 training,
500 validation, and 1{,}000 test puzzles, separated by canonical
arithmetic-structure family. For each training puzzle, the positive is
a verified oracle solution, while negatives are generated without a
learner model and stratified into detail-wrong, near-value,
mid-value, and far-value error bins. The corpus contains 9--16 unique
negative expressions per prompt. A fixed-width representation with
16 negative slots was also materialized for downstream compatibility:
4{,}943 prompts already contained 16 unique negatives, while 1{,}057
prompts were padded by cycling existing exact expressions. No new
synthetic negative is introduced by this conversion. This representation
records conversion provenance rather than the multiplicity used by the
paper-facing objectives. The paper-facing trainers deduplicate exact
expressions and optimize over all unique negatives available for each
prompt. Failure statistics recorded at construction time are for
analysis only: learner-relative near/far status is never stored and is
recomputed only by methods that use the corresponding remoteness
coordinate. The corpus and its replay-pool hash are shared by the
protocol-matched V2 methods; initialization, sampling, calibration,
and ratio-denominator rules are method-specific and are reported in
their corresponding protocol subsections.

\subsection{Training and Evaluation Protocols}

\section{Controlled Experimental Protocols}
\label{app:controlled_protocols}

This section specifies the interventions used in
Sections~\ref{sec:controlled_identification}
and~\ref{sec:repulsion_control}. The definitions of C-U1 and D-U1 are given in
Appendix~\ref{app:experimental_details}. Network architectures,
optimizer settings, training budgets, seeds, and evaluation frequencies
are reported in the protocol-specific subsections below.

\subsection{Shared Notation and Experimental Invariants}
\label{app:controlled_shared}

For a negative sample
$z=(s,a)$, define its raw repulsive update as
\begin{equation}
    g_\theta^-(z)
    =
    -\widehat A^-(z)
    \nabla_\theta
    \log\pi_\theta(a\mid s),
    \qquad
    \widehat A^-(z)
    =
    \max\{-\widehat A(z),0\}.
    \label{eq:controlled_negative_update}
\end{equation}
Learner-relative remoteness is measured by
\begin{equation}
    D_\theta(z)
    =
    -\log\pi_\theta(a\mid s),
    \label{eq:controlled_remoteness}
\end{equation}
or by its equivalent standardized Gaussian component when additive
density constants are irrelevant.

Unless explicitly stated otherwise, all branches of one controlled
comparison share:

\begin{enumerate}
    \item the same training and evaluation contexts;
    \item the same positive and negative samples;
    \item the same fixed rewards or advantage labels;
    \item the same initial policy parameters;
    \item the same minibatch order, optimizer, and training horizon;
    \item the same evaluation checkpoints and terminal classification
    rules.
\end{enumerate}

Only the variable named by the intervention is changed. Actions generated
under the same context are aggregated at the context level and are not
treated as independent contexts for statistical inference.

\subsection{Source-Isolation Protocol}
\label{app:source_isolation_protocol}

\paragraph{Purpose.}
The source-isolation experiment asks whether policy-relative remoteness
changes negative-update magnitude when sample quality is held fixed. It
identifies the source of a per-sample difference, not its causal effect on
training.

\paragraph{Continuous product-manifold construction.}
For each C-U1 context, we select multiple negative actions from a common
reward contour:
\begin{equation}
    R(s,a_i)
    =
    R(s,a_j),
    \qquad
    \widehat A(s,a_i)
    =
    \widehat A(s,a_j)
    <0,
    \label{eq:source_equal_quality}
\end{equation}
while ensuring that
\begin{equation}
    D_\theta(s,a_i)
    \neq
    D_\theta(s,a_j).
    \label{eq:source_different_remoteness}
\end{equation}
The reward coordinate and the policy-distance coordinate are therefore
independently controlled. This exact counterfactual comparison is the
reason for using a synthetic product-manifold construction: an external
offline dataset generally cannot provide two actions with precisely
matched task quality and independently prescribed distance from the
same current policy.

The comparison is performed both at policy initialization and after
positive-only pretraining. No negative sample used as a probe is applied
as an optimizer update during this protocol. Let \(u_\theta(s)\) denote
the policy-output coordinate used by the controlled probe, e.g.,
fixed-covariance Gaussian mean coordinates in C-U1 and direct logits in
D-U1. For every probe, we record
\begin{equation}
\begin{aligned}
I_{\mathrm{score}}(z)
&=
\left\|
\nabla_{u}\log \pi_{u}(a\mid s)
\right\|_2,\\
I_{\mathrm{sample}}(z)
&=
|\widehat A(z)|\, I_{\mathrm{score}}(z).
\end{aligned}
\label{eq:source_influence_metrics}
\end{equation}
Because \(|\widehat A|\) is identical across the matched actions,
differences in \(I_{\mathrm{sample}}\) arise from policy-score geometry
rather than advantage magnitude.

For neural implementations, we additionally record the full
parameter-gradient norm and the aggregate gradient obtained from samples
in the same remoteness range. These implementation-level actor-gradient
diagnostics are used as transfer checks, not as theorem-level score
measures.

We report near-to-far ratios of remoteness, score norm, single-sample
gradient norm, aggregate gradient norm, and gradient-direction
coherence. Ratios are first computed within each context and seed and
are then aggregated across the prespecified seeds.

\paragraph{Categorical counterpart.}
D-U1 uses unordered categorical actions with identical fixed negative
labels. The comparison varies the current probability assigned to the
negative action while preserving its reward role and advantage:
\begin{equation}
    \widehat A(s,i)
    =
    \widehat A(s,j)
    <0,
    \qquad
    \pi_\theta(i\mid s)
    \neq
    \pi_\theta(j\mid s).
    \label{eq:categorical_source_match}
\end{equation}
For a direct categorical logit parameterization, the score norm is
bounded and need not grow without limit as probability decreases.
The relevant far-field effect is instead persistence: the negative
update continues to increase the selected action's logit gap even after
its probability has become small. We therefore report probability,
surprisal, logit gap, score norm, and cumulative suppression over
repeated updates. This distinguishes categorical support suppression
from the unbounded Gaussian-distance amplification measured in C-U1.

\subsection{Causal-Transmission Protocol}
\label{app:causal_transmission_protocol}

\paragraph{Purpose.}
The source-isolation protocol identifies the source of abnormal update
magnitude but not its effect on task behavior. Causal transmission is tested in the nonlinear
Gaussian protocol, where the policy mean and variance evolve jointly,
and in the categorical support reconstruction of D-U1.

\paragraph{Dynamic near--far partition.}
At each intervention step, negative samples are partitioned using their
current-policy remoteness. Let
$\mathcal N_\theta$
and
$\mathcal F_\theta$
denote the prespecified near- and far-field sets. Their membership is
recomputed as the policy evolves rather than frozen at initialization.
The thresholds or quantiles are fixed before training and specified with the
corresponding environment.

\paragraph{Interventions.}
For a negative sample $z$, the following multipliers are applied to
$g_\theta^-(z)$:
\begin{equation}
\begin{aligned}
    w_{\mathrm{uncontrolled}}(z)
    &=1,\\
    w_{\mathrm{positive}}(z)
    &=0,\\
    w_{\mathrm{near\text{-}zero}}(z)
    &=
    \mathbf 1[z\notin\mathcal N_\theta],\\
    w_{\mathrm{far\text{-}zero}}(z)
    &=
    \mathbf 1[z\notin\mathcal F_\theta].
\end{aligned}
\label{eq:causal_zero_interventions}
\end{equation}

The far-cap intervention preserves the direction of a far-field update
but limits its detached magnitude to a reference level estimated from
near-field negative samples:
\begin{equation}
    w_{\mathrm{far\text{-}cap}}(z)
    =
    \begin{cases}
        \displaystyle
        \min\left\{
            1,
            \frac{C_{\mathrm{near}}}
                 {I_{\mathrm{sample}}(z)+\epsilon}
        \right\},
        & z\in\mathcal F_\theta,\\[3mm]
        1,
        & \text{otherwise},
    \end{cases}
    \label{eq:causal_far_cap}
\end{equation}
where $C_{\mathrm{near}}$ is the prespecified near-field update-magnitude
reference.

To distinguish selective intervention from a generic decrease in
negative optimization, the global-matched control applies one scalar
to every negative sample:
\begin{equation}
    w_{\mathrm{global}}(z)
    =
    \alpha_{\mathrm{match}}.
    \label{eq:causal_global_weight}
\end{equation}
The scalar is chosen using the training batch so that its detached aggregate
negative-update magnitude matches the corresponding selective
intervention:
\begin{equation}
    \alpha_{\mathrm{match}}
    \sum_{z\in\mathcal B^-}
        I_{\mathrm{sample}}(z)
    =
    \sum_{z\in\mathcal B^-}
        w_{\mathrm{selective}}(z)
        I_{\mathrm{sample}}(z).
    \label{eq:causal_budget_match}
\end{equation}
The matching coefficient is not differentiated through.

\paragraph{Causal contrasts.}
The primary causal contrast is between removing near-field and
far-field negative updates. If Near-zero leaves the dynamics unchanged
while Far-zero or Far-cap alters them, the effect cannot be attributed
to negative feedback in general. The global-matched branch quantifies
how much of the change follows from reducing total negative-update magnitude
and how much requires selecting samples by remoteness.

\paragraph{Temporal and terminal measurements.}
We record policy displacement, reward, negative-gradient magnitude,
mean dynamics and, for learnable-variance runs, scale dynamics, entropy or effective support, and the first
time at which each prespecified event occurs. The analysis checks whether
far-field update magnitude rises before policy drift and whether drift precedes
task-performance loss.

Terminal outcomes are classified separately as:

\begin{enumerate}
    \item \textbf{task-performance collapse}: a sustained loss of the
    prespecified task metric while optimization remains finite;
    \item \textbf{support or variance-boundary event}: categorical
    effective support or Gaussian variance crosses its prespecified
    boundary without requiring NaN or Inf;
    \item \textbf{numerical collapse}: a parameter, loss, gradient, or
    model output becomes NaN or Inf.
\end{enumerate}

A boundary event is not automatically counted as task-performance
collapse, and finite execution is not automatically counted as
convergence.

\subsection{Controlled Negative-Strength Sweep}
\label{app:negative_strength_sweep}

The controlled negative-strength sweep varies only the effective
negative-repulsion strength $q/p$, where $p$ denotes the aggregate
positive-attraction budget and $q$ denotes the aggregate
negative-repulsion budget. All sweep arms share the same data, fixed
labels, initialization, optimizer, training horizon, and event
thresholds. The positive-only baseline corresponds to $q/p=0$.

Held-out-context reward evaluates unseen-state generalization within the
C-U1 controlled protocol. Policy shift is the normalized displacement of
the learned policy from the positive-only target and visualizes how far
negative feedback pushes the policy beyond the imitation solution.
Task-performance collapse, support or variance-boundary events, and
NaN/Inf numerical failures are reported as separate event types.

\subsection{Controlled Taper Comparison}
\label{app:controlled_taper_protocol}

\paragraph{Controlled utility construction.}
In C-U1, the hidden optimum $a_\star(s)$ lies beyond the center of the
positive demonstrations $a_+(s)$. A local negative anchor is placed on
the opposite side of $a_+(s)$, so repulsion from this anchor points
toward $a_\star(s)$. Additional negative actions are placed on matched
reward contours at larger policy-relative distances. Increasing
negative strength therefore produces three observable regimes:
positive-only under-extrapolation, useful local extrapolation, and
far-field over-extrapolation or instability.

D-U1 implements the same logic with unordered semantic actions. The
hidden optimal target lies beyond the demonstrated positive direction,
while the local negative target lies in the opposite semantic
direction. Repelling the local negative action can therefore improve
the shared representation toward the hidden optimal action. Additional
remote negative actions create support pressure without changing their
fixed negative label. Evaluation on independently sampled contexts is
reported as held-out-context generalization.

Let
\[
x(D)
=
\left[
\frac{D-\tau}{c}
\right]_{+},
\qquad
r_x(D)=\sqrt{x(D)}.
\]
The Linear and Quadratic names refer to polynomial order in the
radialized coordinate \(r_\theta=\sqrt{x_\theta}\), rather than in the
squared-remoteness coordinate \(x_\theta\).
\paragraph{Weighting rules.}
For negative-sample remoteness $D$, the compared rules are
\begin{equation}
\begin{aligned}
    \omega_{\mathrm{positive}}(D)
    &=0,\\
    \omega_{\mathrm{uncontrolled}}(D)
    &=1,\\
    \omega_{\mathrm{global}}(D)
    &=\alpha,\\
    \omega_{\mathrm{hard}}(D)
    &=
    \mathbf 1[D\leq\tau],\\
    \omega_{\mathrm{Rec\text{-}L}}(D)
    &=
    \frac{1}{1+\lambda r_x(D)},\\
    \omega_{\mathrm{Rec\text{-}Q}}(D)
    &=
    \frac{1}{1+\lambda r_x(D)^2},\\
    \omega_{\mathrm{exp}}(D)
    &=
    \exp\{-\lambda r_x(D)^2\}.
\end{aligned}
\label{eq:controlled_taper_rules}
\end{equation}
All weighting values are computed from detached remoteness statistics;
the policy cannot reduce its loss by differentiating through the taper.

\paragraph{Near-field-retention matching.}
Let $\mathcal N$ denote the prespecified useful near-field set and define
its retained first-order update magnitude as
\begin{equation}
    \rho_{\mathrm{near}}(\omega)
    =
    \frac{
        \sum_{z\in\mathcal N}
        \omega(D_\theta(z))
        I_{\mathrm{sample}}(z)
    }{
        \sum_{z\in\mathcal N}
        I_{\mathrm{sample}}(z)
    }.
    \label{eq:near_retention}
\end{equation}
Taper parameters are calibrated on training contexts so that the
compared selective rules retain the same target near-field fraction,
up to the predetermined tolerance. Parameters are then frozen before
evaluation. This comparison isolates how the rules treat the far-field
tail after preserving comparable local utility.

\paragraph{Aggregate-budget matching.}
For a weighting function $\omega$, define its detached negative-update
budget as
\begin{equation}
    B(\omega)
    =
    \sum_{z\in\mathcal B^-}
    \omega(D_\theta(z))
    I_{\mathrm{sample}}(z).
    \label{eq:taper_budget}
\end{equation}
The budget-matched global coefficient is
\begin{equation}
    \alpha_{\mathrm{budget}}
    =
    \frac{
        B(\omega)
    }{
        B(\omega_{\mathrm{uncontrolled}})
    }.
    \label{eq:taper_global_match}
\end{equation}
This control applies the same total first-order negative-update budget
without using remoteness. A selective taper can therefore be credited
only for improvements that remain after comparison with its
budget-matched global baseline.

\paragraph{Measurements.}
Each method is evaluated over the full prespecified training horizon.
We report task performance, held-out-context generalization, distance
from the hidden optimum, policy displacement, entropy or variance,
effective support, near-field retention, aggregate negative-update
budget, and all three terminal-event categories. Method comparisons use
the same seeds and paired statistics.

\subsection{External Taper Validation}
\label{app:external_taper_protocol}

\paragraph{Scientific scope.}
External taper experiments test whether the controlled weighting
principle transfers to complex policies. They do not reproduce the exact
counterfactual geometry of C-U1 or D-U1 and therefore do not replace the
controlled source and causal experiments.

\paragraph{Hopper.}
All Hopper taper methods branch from the same actor initialization and
use the same offline transitions and advantage labels. For an action
with inverse-squash coordinate $u$, Gaussian mean $\mu_\theta(s)$, and
diagonal scale $\sigma_\theta(s)$, the remoteness statistic is
\begin{equation}
    D_{\mathrm{Hopper}}(s,a)
    =
    \frac{1}{2}
    \sum_j
    \left(
        \frac{
            u_j-\mu_{\theta,j}(s)
        }{
            \sigma_{\theta,j}(s)
        }
    \right)^2.
    \label{eq:hopper_taper_remoteness}
\end{equation}
The additive Gaussian normalization term is omitted because it does not
change the ordering used by the taper. Remoteness is recomputed under
the current actor and detached before constructing the negative weight.
The mechanism protocol with frozen advantages is kept separate from the
standard offline-RL performance protocol; the former studies transfer of
the far-field control mechanism, whereas the latter reports final
environment return.

\paragraph{Countdown.}
For each replay prompt, the V2 record contains a verified oracle
solution and 9--16 unique programmatically generated negative
expressions stratified by construction-time error type. A fixed-width
representation with 16 negative slots was also materialized for
downstream compatibility: 4{,}943 prompts already contained 16 unique
negatives, while 1{,}057 prompts were padded by cycling existing exact
expressions. This representation records conversion provenance rather
than the multiplicity used by the paper-facing objectives. The
paper-facing trainers deduplicate exact expressions and optimize over
all unique negatives available for each prompt. Learner-relative
near/far status is never stored and is recomputed only by methods that
use the corresponding remoteness coordinate.

All protocol-matched V2 Countdown methods use the same prompt set,
verified oracle solutions, frozen V2 negative bank, validation set,
and maximum training budget. Initialization, sampling, calibration,
and ratio-denominator rules are method-specific and are stated
separately.

\paragraph{Comparisons and reporting.}
The common external comparisons are Positive-only, Uncontrolled,
Global scaling, and exponential DRPO. Domain-specific baselines are
reported in the overall-performance section rather than treated as
controlled taper variants. Hopper reports normalized return and policy
statistics; Countdown reports greedy verifier success, Pass@$k$, valid
expression rate, surprisal, and support statistics. Task-performance,
boundary, and NaN/Inf outcomes remain separate in both domains.
\subsection{Shared Objectives and Negative-Update Controls}
\label{app:shared_methods}

This section defines the protocol controls and remoteness-aware
weighting rules shared by the D4RL and Countdown experiments. Their
task-specific remoteness signals, data construction, and optimization
protocols are described separately in
Appendices~\ref{app:d4rl_methods} and
\ref{app:countdown_methods}.

\paragraph{Shared signed objective.}

Let \(z\) denote an offline transition or response, let
\(A(z)\) denote its fixed signed learning signal during an actor update,
and let \(D_\theta(z)\) denote its learner-relative remoteness under the
current policy. We write

\[
A^{+}(z)=\max\{A(z),0\},
\qquad
A^{-}(z)=\max\{-A(z),0\}.
\]

The shared family of actor objectives is

\[
\mathcal{J}_{\omega}(\theta)
=
\mathbb{E}_{z\sim\mathcal{B}}
\left[
A^{+}(z)\log \pi_\theta(z)
 -
\omega\!\left(D_\theta(z)\right)
A^{-}(z)\log \pi_\theta(z)
\right],
\]

where positive updates retain unit weight and
\(\omega(D)\) controls only the negative branch. During each actor
update, the advantage and remoteness-derived weight are treated as
fixed coefficients; gradients are not propagated through the weighting
rule.

For the remoteness-aware methods, define the normalized excess
remoteness

\[
x_\theta(z)
=
\left[
\frac{D_\theta(z)-\tau}{c}
\right]_{+},
\]

where \(\tau\) is the near-field threshold, \(c>0\) is the task-specific
scale, and \([u]_{+}=\max\{u,0\}\). Consequently, samples inside the
near-field region \(D_\theta(z)\leq\tau\) retain their full negative
weight.

\paragraph{Positive-only.}

Positive-only removes all negative updates:

\[
\omega_{\mathrm{pos}}(D)=0.
\]

It therefore provides an extreme no-repulsion reference. In D4RL it
retains only transitions with positive actor advantages; in Countdown
it retains only verified successful responses under the corresponding
binary learning signal.

\paragraph{Global-\(\alpha\).}

Global-\(\alpha\) applies the same attenuation to every negative sample:

\[
\omega_{\mathrm{global}}(D)=\alpha,
\qquad
0\leq\alpha\leq 1.
\]

This control tests whether performance improvements can be explained by
a global reduction in negative-update magnitude, without using
learner-relative remoteness.

\paragraph{Reciprocal-Linear and Reciprocal-Quadratic.}
For the remoteness-aware methods, define the normalized excess
remoteness
\[
x_\theta(z)
=
\left[
\frac{D_\theta(z)-\tau}{c}
\right]_{+},
\]
and its radialized coordinate
\[
r_\theta(z)=\sqrt{x_\theta(z)}.
\]
For Gaussian policies, \(r_\theta\) is asymptotically proportional to
standardized distance in the far field; for sequence policies, it is a
surprisal-derived radial coordinate rather than a geometric distance in
the discrete action space.
The two reciprocal controls use the same remoteness signal and
near-field threshold as DRPO, but differ in their far-field decay:

\[
\omega_{\mathrm{Rec\text{-}L}}(D)
=
\frac{1}{1+\lambda r_\theta(z)},
\]

and

\[
\omega_{\mathrm{Rec\text{-}Q}}(D)
=
\frac{1}{1+\lambda r_\theta(z)^2}
=
\frac{1}{1+\lambda x_\theta(z)}.
\]

Reciprocal-Linear has the slowest far-field attenuation. The quadratic
variant decays more rapidly but retains a polynomial tail. These
controls isolate whether the benefit comes merely from using a
remoteness signal or specifically from the decay profile selected by
DRPO.

\paragraph{DRPO.}

DRPO uses an exponential remoteness taper:

\[
\omega_{\mathrm{DRPO}}(D)
=
\exp\left\{
-\lambda x_\theta(z)
\right\}.
\]

The weight equals one throughout the near-field region and then decays
exponentially once remoteness exceeds \(\tau\). Thus, DRPO preserves
local negative feedback while suppressing the remote negative tail more
rapidly than the reciprocal alternatives.

\paragraph{Hyperparameter selection.}

Within each task, \(\tau\), \(c\), \(\lambda\), and the
Global-\(\alpha\) coefficient are selected using the registered
validation protocol only. Test-task performance is not used for
hyperparameter selection. All protocol-matched methods share the same
candidate budgets and checkpoint-selection rule.

\subsection{D4RL Baselines and Experimental Protocol}
\label{app:d4rl_methods}

The definitions of Positive-only, Global-\(\alpha\),
Reciprocal-Linear, Reciprocal-Quadratic, and DRPO follow
Appendix~\ref{app:shared_methods}. This section specifies only the
D4RL-specific baselines, remoteness construction, training protocol,
and evaluation procedure.

\paragraph{Tasks and datasets.}

We evaluate HalfCheetah, Hopper, and Walker2d on the medium,
medium-replay, and medium-expert datasets, producing the nine tasks
reported in Table~\ref{tab:d4rl_performance}. The same fixed offline
dataset is used by every protocol-matched method within a task.

\paragraph{Published offline-RL baselines.}

CQL, TD3+BC, and IQL are included as published reference results.
These values are not protocol-matched reruns and are therefore marked
by \(\dagger\) in Table~\ref{tab:d4rl_performance}. The main table
reports their published task-level means. Their original uncertainty
statistics and exact score provenance are recorded separately from the
protocol-matched DRPO experiments.

\paragraph{Gaussian-policy remoteness.}

For an action \(a\in\mathbb{R}^{d_a}\) and current Gaussian policy with
mean \(\mu_\theta(s)\) and standard deviation \(\sigma_\theta(s)\), we use the
following implementation-level standardized remoteness statistic

\[
D_{\mathrm{D4RL}}(s,a)
=
\frac{1}{d_a}
\sum_{j=1}^{d_a}
\left(
\frac{a_j-\mu_{\theta,j}(s)}
     {\sigma_{\theta,j}(s)}
\right)^2.
\]

This dimension-normalized convention differs from
Equation~\eqref{eq:hopper_taper_remoteness} only by a fixed task-level scale.
The associated threshold and scale parameters absorb that constant, and each
protocol keeps its convention fixed across all compared methods.

When the actor uses a transformed action distribution, this quantity is
computed in the policy's native Gaussian coordinates. The resulting
remoteness is used only to determine the negative-update coefficient;
the distance term is detached during the actor update.

\paragraph{Protocol-matched actor comparison.}

Positive-only, Global-\(\alpha\), Reciprocal-Linear,
Reciprocal-Quadratic, and DRPO share the same offline dataset, critic
checkpoint or critic-training protocol, actor initialization, network
architecture, optimizer, minibatch schedule, actor-update budget,
evaluation schedule, and random seeds. They differ only in the
negative-update weighting rule defined in
Appendix~\ref{app:shared_methods}.

The common actor advantage construction is held fixed across these
methods. No method is allowed to change the critic, relabel samples, or
alter the data distribution in order to improve its actor result.

\paragraph{Hyperparameters.}

The registered training and selection configuration is summarized in
Table~\ref{tab:d4rl_registered_config}. Although the remoteness
coordinate is defined consistently across tasks, its empirical
distribution depends on the dataset. We therefore register
task-specific development grids so that each controller covers
comparable operating regimes under the corresponding remoteness
distribution. All grids and selection rules are frozen before
held-out evaluation and are populated directly from the experiment
configuration, rather than reconstructed or adjusted after observing
test results.

\begin{table}[t]
\centering
\caption{Configuration of the D4RL-9 development sweeps. Each
task-specific grid contains ten candidates: five from the coarse sweep
and five from refinement.}
\label{tab:d4rl_registered_config}

\small
\setlength{\tabcolsep}{6pt}
\renewcommand{\arraystretch}{1.10}

\begin{tabular}{@{}ll@{}}
\toprule
Configuration & Value \\
\midrule
Policy-update budget
& $1{,}000{,}000$ \\
Evaluation
& Every $50{,}000$ updates; $10$ episodes \\
Fixed geometry
& $\tau=0$ \\
Global scaling, $\alpha$
& Task-specific grid \\
Reciprocal coefficients, $c$
& Task-specific grids for Linear and Quadratic \\
Exponential multiplier
& $\alpha_{\mathrm{Exp}}=1$ \\
\bottomrule
\end{tabular}
\end{table}

\paragraph{Evaluation and uncertainty.}

Task performance is reported using the standard D4RL normalized return.
For every protocol-matched method, the main table reports the mean over
the registered seeds. The D4RL-9 total is the sum of
the nine task-level means.

Task-performance collapse is assessed from normalized return and
training trajectories. Variance or support-boundary events are reported
separately from task performance. NaN/Inf numerical failures are
reported as a third outcome and are not merged with either low return or
policy-boundary events. Any claim concerning convergence, collapse, or
method ranking is based on the registered terminal audit rather than an
intermediate checkpoint.

\subsection{Countdown Baselines and Experimental Protocol}
\label{app:countdown_methods}

Positive-only, Global-\(\alpha\), Reciprocal-Linear,
Reciprocal-Quadratic, and DRPO follow the shared definitions in
Appendix~\ref{app:shared_methods}. This section defines AsymRE, the
Joint Fitted-Reference \(\beta\)-TOPR variant, and canonical DPO,
together with the Countdown-specific response-bank, remoteness,
training, and evaluation protocols.

\paragraph{AsymRE.}

Asymmetric REINFORCE defines the response-level advantage using a
tunable reward baseline \(V\):

\[
A_{\mathrm{AsymRE}}(x,y)
=
r(x,y)-V,
\]

and optimizes

\[
\mathcal{J}_{\mathrm{AsymRE}}(\theta)
=
\mathbb{E}_{(x,y)\sim\mathcal{B}}
\left[
\bigl(r(x,y)-V\bigr)
\log\pi_\theta(y\mid x)
\right].
\]

For a binary verifier reward, lowering \(V\) emphasizes successful
responses, whereas increasing \(V\) strengthens suppression of failed
responses. In our branch-balanced signed-reward implementation,
\(r\in\{-1,+1\}\) gives the empirical baseline
\(\widehat V=0\). The swept offset therefore determines the
objective baseline as
\(V=\widehat V+\delta_v=\delta_v\). The baseline is selected using the validation protocol and
is held fixed within each run.

Global-\(\alpha\) and AsymRE remain separate comparisons because they
operate on different underlying learning signals. Global-\(\alpha\)
scales the negative part of the common precomputed signed advantage,
whereas AsymRE constructs its signed advantage directly from the raw
verifier reward and the published baseline parameterization. Under
binary rewards and simplified normalization they are closely related,
but the formal comparison preserves their native objectives and tuning
rules.

\paragraph{\(\beta\)-TOPR (joint fitted reference).}

Let \(\mathcal X\) be the training-prompt set. For each
\(x\in\mathcal X\), let \(y_x^+\) be the verified oracle solution and
let \(\mathcal N_x\) be the set of unique negative expressions after
exact-expression deduplication. Define the mean completion-token log
probability

\[
\bar{\ell}_\theta(x,y)
=
\frac{1}{|y|}
\sum_{t=1}^{|y|}
\log\pi_\theta(y_t\mid x,y_{<t}),
\]

and the full-completion summed log probability used only for the ratio,

\[
s_\theta(x,y)
=
|y|\,\bar{\ell}_\theta(x,y)
=
\sum_{t=1}^{|y|}
\log\pi_\theta(y_t\mid x,y_{<t}).
\]

For the jointly fitted reference adapter \(\mu_{\mathrm{ref}}\), define

\[
s_{\mathrm{ref}}(x,y)
=
\sum_{t=1}^{|y|}
\log\mu_{\mathrm{ref}}(y_t\mid x,y_{<t}),
\]

\[
\rho_\theta(x,y)
=
\exp\!\left(s_\theta(x,y)-s_{\mathrm{ref}}(x,y)\right),
\]

and

\[
\begin{aligned}
w_\beta(x,y)
&=
\exp\!\left(
\beta\min\{s_\theta(x,y)-s_{\mathrm{ref}}(x,y),0\}
\right) \\
&=
\min\{\rho_\theta(x,y)^\beta,1\}.
\end{aligned}
\]

The policy objective is

\[
\begin{aligned}
\mathcal J_{\beta\text{-}\mathrm{TOPR}}(\theta)
={}&
\mathbb E_{x\sim\operatorname{Unif}(\mathcal X)}
\Bigg[
\bar{\ell}_\theta(x,y_x^+) \\
&-
\frac{1}{|\mathcal N_x|}
\sum_{y\in\mathcal N_x}
\operatorname{sg}[w_\beta(x,y)]
\bar{\ell}_\theta(x,y)
\Bigg].
\end{aligned}
\]

with actor update

\[
g_{\beta\text{-}\mathrm{TOPR}}(\theta)
=
\nabla_\theta\mathcal J_{\beta\text{-}\mathrm{TOPR}}(\theta).
\]

This objective samples prompts uniformly, assigns unit total
coefficient to the positive branch and unit total coefficient to the
negative branch, and averages unique negatives within each prompt. It
is not a simple uniform average over all response rows. The
ratio-derived weight is detached from the policy gradient.

The reference adapter is fitted with

\[
\mathcal J_{\mathrm{ref}}
=
\mathbb E_{x\sim\operatorname{Unif}(\mathcal X)}
\left[
\frac{1}{2}\bar{\ell}_{\mathrm{ref}}(x,y_x^+)
+
\frac{1}{2|\mathcal N_x|}
\sum_{y\in\mathcal N_x}
\bar{\ell}_{\mathrm{ref}}(x,y)
\right].
\]

The reference adapter is updated once per policy step. Reference
outputs and ratio-derived weights are detached from the policy
gradient. The 0.5/0.5 masses belong only to the reference-fitting
objective, not to the policy objective. The data records do not contain
behavior-policy log probabilities; the ratio denominator is supplied by
the jointly fitted reference adapter during training. The completed
TOPR coefficient scan is a provenance-limited development pilot and does
not establish convergence, statistical significance, steady state, or
a formal method ranking.

\paragraph{Canonical DPO.}
Canonical DPO~\citep{rafailov2023direct} uses the same LoRA
parameterization as the other Countdown methods and is preceded by a
short oracle-positive SFT stage so that preference optimization starts
from a task-capable policy rather than a near-random solver. The
resulting SFT LoRA checkpoint reaches held-out Pass@8 in the
\(0.10\)--\(0.14\) range and initializes both the trainable policy
adapter and the frozen reference adapter. For each prompt \(x\), the verified oracle solution
\(y_x^+\) is paired with every unique verifier-incorrect completion
\(y\in\mathcal N_x\). Using the summed completion log probability
\(s_\theta\) defined above, the loss is

\[
\begin{aligned}
\mathcal L_{\mathrm{DPO}}(\theta)
={}&-
\mathbb E_{x\sim\operatorname{Unif}(\mathcal X)}
\Bigg[
\frac{1}{|\mathcal N_x|}
\sum_{y\in\mathcal N_x}
\log\sigma\Bigl(
\beta\bigl[
 s_\theta(x,y_x^+)-s_{\mathrm{ref}}(x,y_x^+) \\
&\hspace{8.2em}
-s_\theta(x,y)+s_{\mathrm{ref}}(x,y)
\bigr]
\Bigr)
\Bigg].
\end{aligned}
\]

The loss is averaged over unique negatives within each prompt and then
uniformly across prompts. The reference remains frozen throughout the
fixed 1,200-step continuation, and \(\beta\) is selected using the
registered validation sweep.

\paragraph{Method-specific remoteness.}

For DRPO and the other remoteness-taper controls, the learner-relative
coordinate for a prompt \(x\) and completion
\(y=(y_1,\ldots,y_T)\) is the current mean completion-token surprisal

\[
D_{\mathrm{CD}}(x,y)
=
-\frac{1}{T}
\sum_{t=1}^{T}
\log
\pi_\theta
\left(
y_t\mid x,y_{<t}
\right).
\]

Only completion tokens are included; prompt tokens and padding positions
are excluded. Mean rather than summed surprisal is used so that response
length does not mechanically determine DRPO remoteness. The remoteness
score is detached before applying the shared taper functions. AsymRE
does not use remoteness and applies its signed reward directly. The
Joint Fitted-Reference \(\beta\)-TOPR variant uses the dynamic
full-completion policy/reference log-ratio defined above. Canonical DPO
instead optimizes the frozen-reference pairwise log-ratio in its own
objective; neither baseline uses DRPO's mean-token remoteness.

\paragraph{Frozen response-bank construction.}

Before policy optimization, the registered model-independent generator
constructs a frozen prompt-indexed V2 bank. Each training prompt has one
verified oracle solution and 9--16 unique programmatically generated
negative expressions stratified by construction-time error type. The
paper-facing trainers deduplicate exact expressions and use all unique
negatives available for each prompt. The bank stores no learner-relative
near/far label and no behavior-policy log probability. Verifier outcome,
response length, syntax status, number-use status, and arithmetic result
remain separate stored diagnostics even when they share the same binary
training reward.

\paragraph{Training and checkpoint selection.}

AsymRE, the Joint Fitted-Reference \(\beta\)-TOPR variant, and DRPO
initialize their trainable policy LoRA adapters directly from the same
Qwen2.5-0.5B-Instruct backbone~\citep{qwen2024qwen25}. The TOPR arm
additionally creates a reference LoRA copied from its initial policy
adapter. Canonical DPO uses the same LoRA parameterization but instead
initializes both its trainable policy adapter and permanently frozen
reference adapter from the short SFT LoRA checkpoint before the first
DPO update. The registered continuation
methods share the AdamW optimizer~\citep{loshchilov2019decoupled},
learning-rate schedule, minibatch construction, number of optimizer
steps, and validation split unless a method-specific protocol states
otherwise. Hyperparameters are selected using the registered validation
rule.

The registered training and selection configuration is summarized in
Table~\ref{tab:countdown_registered_config}. All entries must be
populated directly from the frozen experiment configuration rather than
reconstructed after observing the test results.

\begin{table}[t]
\centering
\caption{Registered configuration for the Countdown performance experiments.}
\label{tab:countdown_registered_config}

\footnotesize
\setlength{\tabcolsep}{4pt}
\renewcommand{\arraystretch}{1.10}

\begin{tabular}{@{}p{0.38\columnwidth}p{0.56\columnwidth}@{}}
\toprule
Configuration & Registered value \\
\midrule

Training horizon
& $1{,}200$ optimizer steps; no early stopping \\

Update configuration
& Micro-batch $1$, gradient accumulation $8$; AdamW with learning rate
$5\times10^{-5}$, weight decay $0.01$, cosine decay, and $3\%$ warmup \\

AsymRE
& Selected $V=\delta_v=-1$ \\

Joint fitted-reference TOPR
& Selected $\beta=0.25$ \\

Canonical DPO
& One-epoch SFT warm start; selected $\beta=1.0$ \\

DRPO
& $\tau=0$, $c=2$, and selected $\lambda=1.897119985$ \\

Global reference
& $\alpha=0.03125$ ($1/32$) \\

Evaluation
& Pass@8 primary; Pass@64 supplementary \\

\bottomrule
\end{tabular}
\end{table}

\paragraph{Verifier-based evaluation.}

Greedy verifier success is the fraction of held-out puzzles solved by
greedy decoding. Pass@$k$ is the fraction solved by at least one of
the registered \(k\) sampled responses. Valid-expression rate is the
fraction of generated responses that satisfy the parser and number-use
constraints, irrespective of whether they reach the target.

These three metrics are reported separately. In particular, a higher
verifier success rate is not attributed to improved reasoning when it
is accompanied by a deterioration in expression validity.

Task-performance degradation, token- or sequence-support boundary
events, and NaN/Inf numerical failures are audited independently.
Claims about stability or final method ranking use the registered
terminal checkpoint audit rather than finite-step pilot behavior.

\subsection{Countdown Baseline-Coefficient Sensitivity}
\label{app:countdown_baseline_sensitivity}

Figure~\ref{fig:app_countdown_baseline_parameter_response} reports the
complete parameter-response sweeps of AsymRE and the Joint
Fitted-Reference \(\beta\)-TOPR variant. For AsymRE, performance is
maximized at \(\delta_v=-1\) and drops substantially as \(\delta_v\)
increases, placing the strongest tested finite-horizon configuration at
the zero-negative boundary. For joint fitted-reference \(\beta\)-TOPR,
\(\beta=0\) performs poorly, while positive \(\beta\) values recover
performance and form a broad plateau rather than a sharply isolated
optimum. We therefore use \(\delta_v=-1\) for AsymRE and \(\beta=0.25\)
for joint fitted-reference \(\beta\)-TOPR in the main comparison.
Across the registered \(\beta\) sweep, warm-started canonical DPO
fails to preserve its initial task capability and collapses under
every tested setting.

\begin{figure}[t]
    \centering

    \begin{minipage}[t]{0.49\columnwidth}
        \centering
        \vspace{0pt}
        \parbox[t][2.8em][t]{\linewidth}{
            \centering
            \textbf{(a) AsymRE}
        }
        \includegraphics[
            width=\linewidth
        ]{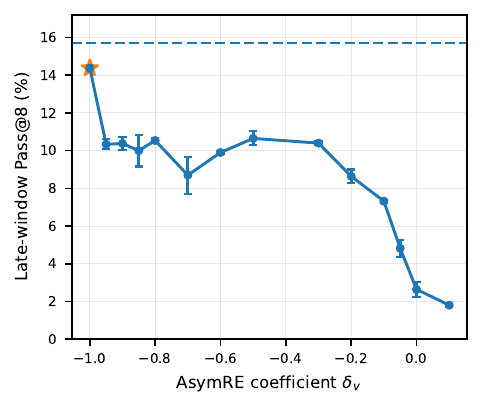}
    \end{minipage}
    \hfill
    \begin{minipage}[t]{0.49\columnwidth}
        \centering
        \vspace{0pt}
        \parbox[t][2.8em][t]{\linewidth}{
            \centering
            \textbf{(b) Joint fitted-reference}\\[-0.1em]
            \textbf{\(\beta\)-TOPR}
        }
        \includegraphics[
            width=\linewidth
        ]{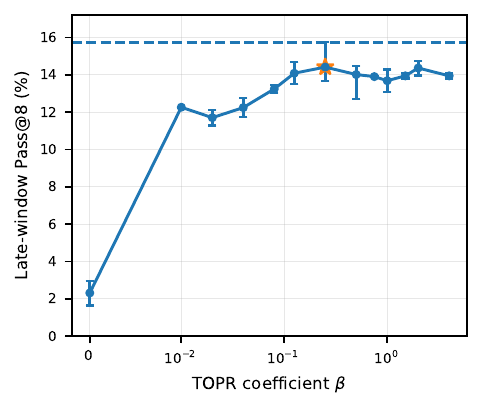}
    \end{minipage}

    \vspace{-0.35em}
    \caption{
        \textbf{Countdown-0.5B baseline-coefficient response.}
        \textbf{(a)} AsymRE late-window Pass@8 versus \(\delta_v\).
        \textbf{(b)} Joint fitted-reference \(\beta\)-TOPR late-window
        Pass@8 versus \(\beta\). Points denote means over the available
        development trajectories; error bars span their minimum--maximum
        range. The dashed line denotes DRPO at
        \(\lambda=1.897\) with \(c=2\), and stars mark the baseline
        settings used in Table~\ref{tab:countdown_performance}:
        \(\delta_v=-1\) and \(\beta=0.25\).
    }
    \label{fig:app_countdown_baseline_parameter_response}
\end{figure}

\subsection{Countdown Taper-Coefficient Sensitivity}
\label{app:countdown_taper_sensitivity}

Figure~\ref{fig:app_countdown_taper_coefficient_response} reports the
complete stored late-window Pass@8 coefficient response for the
Countdown-0.5B taper study under the fixed 1,200-step budget. Here the
late-window metric averages the final five evaluations (steps 800, 900,
1000, 1100, and 1200). Each point is the mean of the two paired
development seeds, and each error bar spans the corresponding pair of
seed values. The horizontal reference is the Positive-only late-window
mean of 14.0\%.

\begin{figure}[t]
    \centering
    \includegraphics[
        width=\columnwidth
    ]{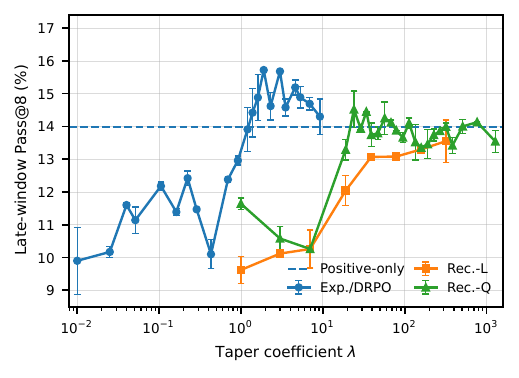}
    \caption{
        \textbf{Countdown-0.5B taper-coefficient response.}
        Late-window Pass@8 is shown against the taper coefficient $\lambda$
        under the fixed 1,200-step budget. Curves report two-seed late-window means;
        error bars span the paired seed values, and the dashed horizontal
        line denotes Positive-only. Exp./DRPO is shown over
        $\lambda\in[0.01,9.2103]$, Reciprocal-Linear through $319$, and
        Reciprocal-Quadratic through $1279$.
    }
    \label{fig:app_countdown_taper_coefficient_response}
\end{figure}

The exponential response rises above Positive-only over an intermediate
coefficient region, reaches a maximum two-seed late-window mean of 15.7\%,
and returns toward the Positive-only level at the largest evaluated
coefficients. Reciprocal-Linear increases toward a broad plateau but
remains slightly below Positive-only, with a maximum late-window mean of
13.6\%. Reciprocal-Quadratic reaches 14.5\% at $\lambda=24$ and then
fluctuates around a broad reference-level plateau through $\lambda=1279$.
These complete curves show that the observed difference is a
response-shape effect rather than a comparison of isolated best points.
Because this study uses two paired development seeds and a fixed horizon,
it supports a bounded tail-shape sensitivity claim, not convergence,
statistical significance, or a universal method ranking.

\clearpage
\end{document}